\documentclass{article} 
\usepackage{iclr2021_conference,times}

\usepackage{amsmath,amsfonts,bm}

\def\eqref#1{equation~\ref{#1}}
\def\Eqref#1{Equation~\ref{#1}}

\def\1{\bm{1}}

\DeclareMathAlphabet{\mathsfit}{\encodingdefault}{\sfdefault}{m}{sl}
\SetMathAlphabet{\mathsfit}{bold}{\encodingdefault}{\sfdefault}{bx}{n}

\newcommand{\R}{\mathbb{R}}

\DeclareMathOperator*{\argmin}{arg\,min}

\usepackage{hyperref}
\usepackage{url}

\usepackage[utf8]{inputenc} 
\usepackage[T1]{fontenc}    
\usepackage{booktabs}       
\usepackage{amsfonts}       
\usepackage{nicefrac}       
\usepackage{microtype}      
\usepackage{amsmath}
\usepackage{algorithm}
\usepackage{algorithmic}
\usepackage{multirow}
\usepackage{graphicx}
\usepackage{epstopdf}
\usepackage{float} 
\usepackage{subfigure}
\usepackage{enumitem}
\usepackage{color}
\usepackage{booktabs}

\usepackage{amsthm, amssymb, mathtools}
\usepackage{soul}
\usepackage{algorithm}
\usepackage{algorithmic}
\usepackage{mathrsfs}
\usepackage{caption}
\usepackage{tabularx}

\newtheorem{theorem}{Theorem}

\newtheorem{lemma}{Lemma}
\theoremstyle{remark}

\title{Adaptive Gradient Method with Resilience and Momentum}

\author{Jie Liu$^{1}$, Chen Lin$^{1}$, Chuming Li$^1$, Ming Sun$^1$, Junjie Yan$^1$\\
$^1$Sensetime Research Group\\
\texttt{\{liujie4,linchen\}@sensetime.com} \\
\And
Lu Sheng$^2$ \\
$^2$Beihang University\\
\texttt{lsheng@buaa.edu.cn}
\And
Wanli Ouyang$^3$ \\
$^3$The University of Sydney\\
\texttt{wanli.ouyang@sydney.edu.au}
}

\begin{document}
\maketitle
\begin{abstract}

Several variants of stochastic gradient descent (SGD) have been proposed to improve the learning effectiveness and efficiency when training deep neural networks, among which some recent influential attempts would like to adaptively control the parameter-wise learning rate (\emph{e.g.}, Adam and RMSProp).
Although they show a large improvement in convergence speed, most adaptive learning rate methods suffer from compromised generalization compared with SGD.
In this paper, we proposed an Adaptive Gradient Method with Resilience and Momentum (AdaRem), motivated by the observation that the oscillations of network parameters slow the training, and give a theoretical proof of convergence.
For each parameter, AdaRem adjusts the parameter-wise learning rate according to whether the direction of one parameter changes in the past is aligned with the direction of the current gradient, and thus encourages long-term consistent parameter updating with much fewer oscillations.
Comprehensive experiments have been conducted to verify the effectiveness of AdaRem when training various models on a large-scale image recognition dataset, \emph{i.e.}, ImageNet, which also demonstrate that our method outperforms previous adaptive learning rate-based algorithms in terms of the training speed and the test error, respectively.

\end{abstract}

\section{Introduction}

Stochastic gradient descent based optimization methods, \emph{e.g.}, SGD~\citep{robbins1951stochastic}, have become the most popular algorithms to train deep neural networks, especially used in high-level vision tasks such as image recognition~\citep{hu2018squeeze}, object detection~\citep{song2020revisiting}, instance segmentation ~\citep{dai2016instance} and etc.

However, as stated in~\citet{luo2019adaptive}, one limitation of SGD is that it uniformly scales each element in the gradient of one network parameter.
Recent efforts have discovered a variety of adaptive methods that rescale the gradient based on the element-wise statistics.
Referred as adaptive learning rate family, these methods include Adagrad~\citep{duchi2011adaptive}, Adadelta \citep{zeiler2012adadelta}, Adam~\citep{kingma2014adam}, AdaBound ~\citep{luo2019adaptive} and RMSProp \citep{tieleman2012divide}.

In particular, Adagrad, which firstly proposes to adaptively modify the learning rates, was later found to have poor performance because of the rapid decay of the learning rate. Many variants of Adagrad, such as RMSProp, Adam, Adadelta, Nadam were proposed to solve this problem by adopting an exponential moving average. Among these adaptive optimization methods, Adam has been widely used in the community due to its fast convergence. 
Despite its popularity, \citet{wilson2017marginal} recently found that the generalization performance of Adam and its variants is worse than their non-adaptive counterpart: SGD with momentum (SGDM) and weight decay~\citep{krogh1992simple}, even when better memorization on the train set is observed. 

From a different viewpoint towards the adaptive learning rate methods, we investigate the trajectories of model parameters in the training process, in which a lot of oscillations are observed for each parameter.
It seriously hinders the network training. Oscillations might come from the evaluation at random subsamples (mini-batches) of data points, or arise from the inherent function that changes dramatically locally. To address this issue, we propose Adaptive Gradient Methods with REsilience and Momentum (AdaRem), a new adaptive optimization method that reduces useless oscillations by introducing damping. For each parameter, AdaRem adjusts the learning rate according to whether the direction of the current gradient is the same as the direction of parameter change in the past. Furthermore, we find that weight decay affects the magnitude of the gradient, thus affecting the estimation of the update direction of the parameter.
Inspired by \citet{li2019exponential}, we also propose AdaRem-S, a variant of AdaRem that constrains the optimization of the neural network to a sphere space so as to eliminate the influence of weight decay. 

Adaptive learning rate methods are usually hard to tune in practical scenarios. However, we find that AdaRem and AdaRem-S usually perform well by simply borrowing SGDM's hyper-parameters, which tremendously reduces the burden of hyper-parameter tuning.
Furthermore, adaptive methods suffer "the small learning rate dilemma"~\citep{chen2018closing}, which means adaptive gradient methods have to choose a very small base learning rate to allieviate over-large learning rates on some coordinates. However, after several rounds of decaying, the learning rates are too small to make any significant progress in the training process~\citep{chen2018closing}. %
With a learning rate as large as that used in SGD, our methods avoid this dilemma.

Finally, we evaluate AdaRem and AdaRem-S by training classifiers on the ImageNet~\citep{deng2009imagenet} dataset. Most of the previous works about optimizers only evaluate their methods on small datasets such as CIFAR-10~\citep{krizhevsky2009learning}. We argue that these datasets are not enough representative to fairly and comprehensively measure the performance of optimizers for nowadays vision tasks. Therefore, we conduct all experiments directly on ImageNet with various models. Experimental results show that AdaRem (including its spherical version) has higher training speed and at the meantime leads to improved performance on the test datasets compared to existing popular adaptive methods.
We show that AdaRem and AdaRem-S close the performance gap between adaptive gradient methods and SGDM empirically.
Furthermore, AdaRem-S can bring considerable improvement over SGDM in terms of the final performance, especially on small networks. 

In summary, our main contributions are the following:
\begin{itemize}[leftmargin=*]
\item We propose AdaRem, a novel and simple adaptive optimization method that accelerates the training process by reducing useless oscillations, which enjoys a fast convergence speed and performs as well as SGDM on the unseen data.
\item We improve our method by constraining the optimization on a sphere space. The resulted variant: AdaRem-S shows significant improvement on top-1 accuracy for MobileNetV2 \citep{sandler2018mobilenetv2} and ShuffleNetV2 \citep{ma2018shufflenet} on ImageNet.
\end{itemize}

\section{Related work}
Optimization methods directly related to AdaRem are Rprop and Momentum. Their similarities and differences are discussed below. Other adaptive optimization methods mainly include Adam and its variants. \citet{kingma2014adam} proposes Adam which is particularly popular on vision tasks. \citet{dozat2016incorporating} increases the performance by combining Adam and the Nesterov accelerated gradient. AdamW \citep{loshchilov2017decoupled} attempts to recover the original formulation of weight decay regularization by decoupling the weight decay from the gradient updates and thus substantially improves Adam’s generalization performance. Recently, \citet{reddi2019convergence} observes that Adam does not converge in some settings due to the “short memory” problem of the exponential moving average. They fix this problem by endowing Adam with “long-term memory” of past gradients and propose Amsgrad optimizer. Based on Amsgrad, a series of modified Adam optimization methods emerge, including PAdam \citep{chen2018closing}, AdaShift \citep{zhou2018adashift}, AdaBound \citep{luo2019adaptive}, NosAdam \citep{huang2018nostalgic}.

\paragraph{RProp}
A closely related adaptive optimization method is Rprop \citep{riedmiller1992rprop}. There are a few important differences between Rprop and AdaRem: Rprop increases or decreases each learning rate ${\eta}_i$ according to whether the gradient concerning ${w}_i$ has changed sign between two iterations or not, whereas AdaRem adjusts ${\eta}_i$ according to whether a running average of the history gradient has an opposite sign with the gradient of current iteration or not. What’s more, AdaRem changes the learning rate softly while Rprop changes the learning rate multiplicatively (e.g. times 1.2). Furthermore, Rprop does not work with mini-batches \citep{tieleman2012divide}.

\paragraph{Momentum}
Momentum \citep{sutskever2013importance} helps accelerate SGD in the relevant direction and dampens oscillations. However, it directly uses the momentum term as the damping, without considering whether the damping term is in the same direction as the current gradient and adaptively adjust the learning rate.

\section{Method}
In this section, we introduce our AdaRem and AdaRem-S methods. Firstly, we offer the notations and preliminaries. Secondly, we explain the motivations of our approach. Then we introduce our momentum guided adaptive gradient method in detail. Finally, the spherical version of our method called AdaRem-S is presented.

\subsection{Notations and preliminaries}
Given two vectors $v, v \in \R^d$, we use $\langle v, v \rangle$ for inner product, $v \odot v$ for element-wise product and $v/v$ to denote element-wise division. 
For the set of all positive definite $d \times d$ matrices. We use $\mathcal{S}_+^d$ to denote it.
For a vector $a \in \R^d$ and a positive definite matrix $A \in \R^{d \times d}$, we use $a/A$ to denote $A^{-1}a$ and $\sqrt{A}$ to denote $A^{1/2}$.
The projection operation $\Pi_{\mathcal{F}, A}(y)$ for $A \in \mathcal{S}_+^d$ is defined as $\argmin_{x \in \mathcal{F}} \|A^{1/2}(x-y)\|$ for $y \in \R^d$.
Furthermore, we say $\mathcal{F}$ has bounded diameter $D_\infty$ if $\|x-y\|_\infty \leq D_\infty$ for all $x, y \in \mathcal{F}$.

Scalars and vectors are denoted in lowercase and bold lowercase, respectively. Our goal is to solve the optimization problem:
\(\boldsymbol{\theta}^{*}=\arg \min _{\boldsymbol{\theta} \in \mathbb{R}^{n}} f(\boldsymbol{\theta}).\)
We denote the gradient with $\boldsymbol{g}_{t}=\nabla_{\boldsymbol{\theta}} f(\boldsymbol{\theta})$ at timestep $t$. We use $g_{t,i}$ to represent the $i^{th}$ component of vector $\boldsymbol{g_t}$ and $\left\|\boldsymbol{\theta}_{t}\right\|$ to represent the length of vector $\boldsymbol{\theta}_{t}$.
Consider the general formula of adaptive optimization methods: 
\(\boldsymbol{\theta}_{t+1}=\boldsymbol{\theta}_{t}-\eta_{t}\boldsymbol{a}_{t}\odot \boldsymbol{g}_{t}, \label{3.1-1}\)
where $\eta_t$ is the learning rate at timestep $t$ and $\odot$ is elementwise multiplication. We use $\boldsymbol{a}_{t}$ to adjust $\eta_{t}$ adaptively.

Following  \citep{reddi2019convergence}, we use online convex programming framework to analyze our optimization methods.
It can be formulated as a repeated game between a player and an adversary. At iteration t, the player chooses $\boldsymbol{\theta}_t$ from convex set $\mathcal{F}$ as learned parameters of the model. Then the adversary chooses a convex loss function $f_t$ which can be seemed as the loss of the model with the chosen parameters and the next minibatch's data. The method's regret at the end of $T$ iterations of this process is given by \(R_T = \sum_{t=1}^T f_t(\boldsymbol{\theta}_t) - \min_{\boldsymbol{\theta} \in \mathcal{F}} \sum_{t=1}^T f_t(\boldsymbol{\theta})\), where the former term is the total loss and the latter term is the smallest total loss of any fixed parameters. 
Throughout this paper, we assume that the feasible set $\mathcal{F}$ has bounded diameter and $\|\nabla f_t(\boldsymbol{\theta})\|_\infty$ is bounded for all $t \in [T]$ and $\theta \in \mathcal{F}$. Our aim is to devise an algorithm that ensures $R_T = o(T)$, which implies that on average, the model's performance converges to the optimal one.

\subsection{Motivation}
\subsubsection{Key observation}
As a powerful technique to reduce oscillations during training, momentum has become the standard configuration of SGD. However, as shown in Figure \ref{Fig.mtd0}, we observe that the oscillation is still severe for each parameter. We argue that directly using the momentum instead of the gradient to reduce the oscillation in SGDM does not give full play to the effect of momentum. 
\subsubsection{The quantitative index of oscillations}
 To better understand the oscillation of various optimization methods in the training process, we propose a quantitative metric to evaluate it. As seen in Figure \ref {Fig.mtd0}, we record the path length and displacement from the starting point to the ending point for each parameter and define “the path length per displacement”: 
\(q=l/d\)
, where path length $l$ is the total distance a parameter travels from a starting point and displacement $d$ is the shortest distance between the ending point and starting point of a parameter. Assuming the number of parameters is $n$, We define the average value of $q$ among all parameters:
\[Q=\frac{1}{n} \sum_{i=1}^{n} q_{i}\]to measure the degree of oscillation in the training process.
As seen in Figure \ref {Fig.0}, the $Q$ value of SGDM is significantly smaller than SGD verifying the effect of momentum on oscillation reduction.
Furthermore, Adam has a smaller $Q$ than SGDM indicating an adaptively changing learning rate may be useful to suppress oscillations. 

\begin{figure}[h] 
\centering 
\begin{minipage}[b]{0.46\textwidth} 
\centering 
\includegraphics[width=0.8\textwidth]{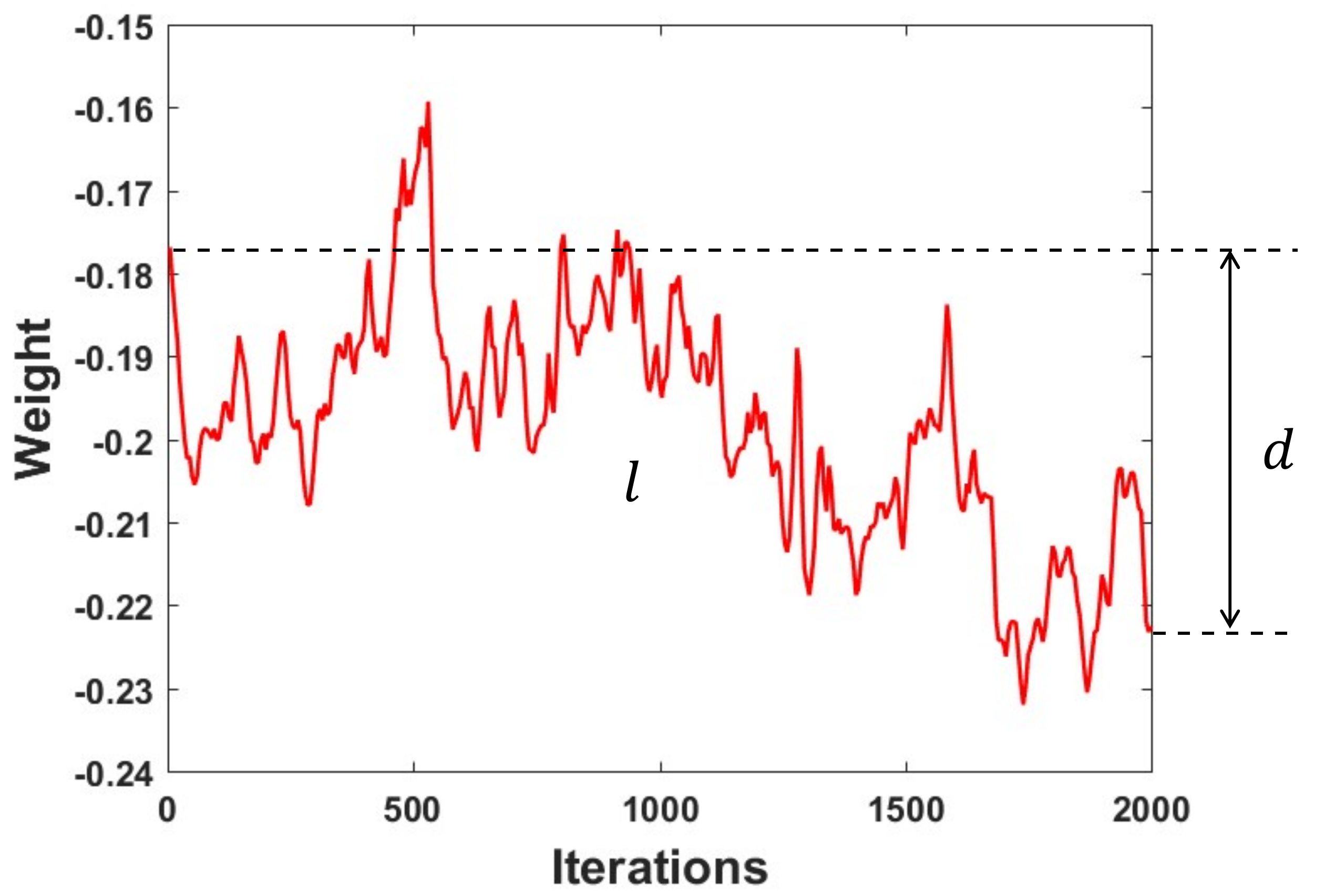}
\caption{The parameter changes as training is going on. From Iteration 0 to Iteration 2k, the effective displacement of this parameter is $d$(about 0.04), but the path length $l$ is observably much larger than $d$.}
\label{Fig.mtd0}
\end{minipage}
\hspace{9mm}
\begin{minipage}[b]{0.46\textwidth} 
\centering 
\includegraphics[width=0.76\textwidth]{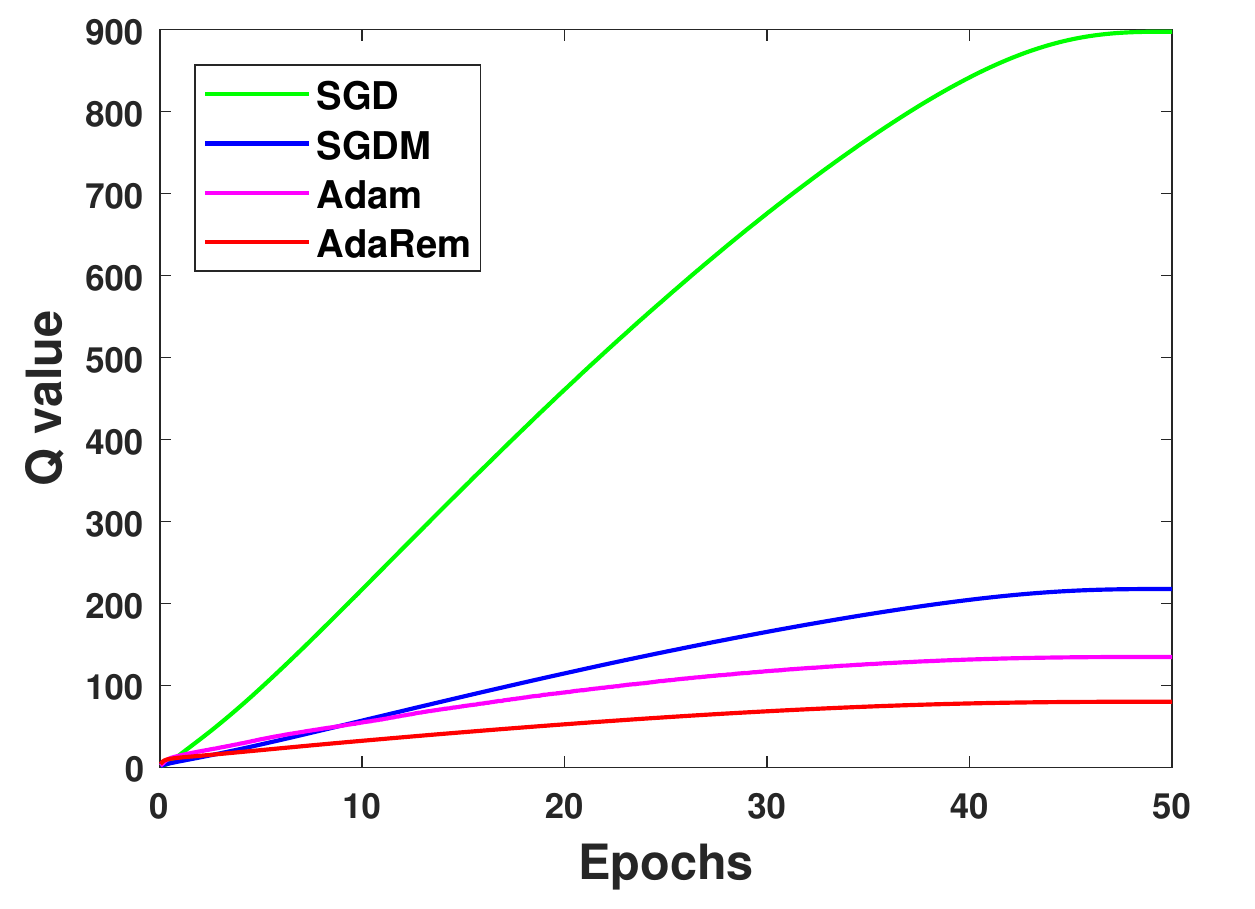}
\caption{The Q value of AdaRem, Adam and SGDM for ResNet-18 on ImageNet. Higher Q value means more oscillations and useless updates. Compared with SGDM and Adam, AdaRem reduces oscillation more effectively.}
\label{Fig.0}
\end{minipage}
\end{figure}

\subsubsection{Adaptive gradient with resilience and momentum}
Consider an elastic ball rolling down a very rough road from a high place, with a lot of oscillations going back and forth as it falls. If the ball can remember the update direction of each coordinate and change the learning rate based on whether the current update direction is consistent with "memory", it will descend more smoothly and fastly. We propose a similar adaptive way to employ momentum with resilience, which adjusts the learning rate according to whether the direction of momentum is the same as the direction of the current gradient for each parameter.
As shown in Figure \ref {Fig.0}, our method AdaRem has the smallest $Q$ value which means AdaRem can accelerate SGD, dampen oscillations and reduce useless updates during training more efficiently than SGDM and Adam.

\subsection{AdaRem}
In this section, we describe the AdaRem algorithm and discuss its properties. 
The algorithm maintains an exponential moving average(EMA) of the gradient($\boldsymbol{m}_t$) where the hyper-parameter $\beta \in[0,1)$ control the exponential decay rate. The moving average is an estimation of the mean of the gradient, we call it momentum, whose $i^{th}$ component(${m}_{t,i}$) can represent the update trend(increase or decrease) of the corresponding parameter of neural network in the past. 
An important property of AdaRem's update rule is its special adjustment rule of learning rate:
\begin{equation}
{b}_{t,i}=\frac{{g}_{t,i} \times {m}_{t,i}}{\left|{g}_{t,i}\right| \max \left|\boldsymbol{m}_{t}\right|+\epsilon},
\end{equation}
\vspace{-1mm}
\begin{equation}
\boldsymbol{a}_{t}=1+\lambda^t\boldsymbol{b}_{t},
\end{equation}
where $\epsilon$ is a term added to the denominator to improve numerical stability. The greater the component of the momentum $\boldsymbol{m}_{t}$, the greater the damping of this term. 
The adjustment coefficient $\boldsymbol{a_t}$ is bounded: ${b}_{t,i} \in[-1,1]$, and then ${a}_{t,i} \in[1-\lambda,1+\lambda]$, which can obviate very large learning rates on some coordinates and escape from "the small learning rate dilemma" \citep{chen2018closing}.
For each parameter, AdaRem adjusts the learning rate according to whether the direction of the current gradient is the same as the direction of parameter change in the past. There are two cases:
\begin{enumerate}
\item[(1)] ${g}_{t,i} \times {m}_{t,i} \geq 0, {b}_{t,i} \in[0,1], {a}_{t,i} \in[1,1+\lambda]$. The direction of the current gradient is the same as the direction of parameter change in the past. Therefore, the current update should be encouraged. We use $\left|{g}_{t,i}\right| \max \left|\boldsymbol{m}_{t}\right|$ to normalize ${b}_{t,i} \in[-1,1]$;
\item[(2)] ${g}_{t,i} \times {m}_{t,i}<0, {b}_{t,i} \in[-1,0), {a}_{t,i} \in[1-\lambda,1)$. The direction of the current gradient is opposite to the direction of parameter change in the past. Therefore, the current update should be suppressed. 

\end{enumerate}
An important property of AdaRem is that the gradient does not change the sign in spite of a large momentum with opposite sign. This means AdaRem is friendly to the situation where the gradient changes dramatically. 

Following \citet{reddi2019convergence} and \citet{luo2019adaptive}, we analyze the convergence of AdaRem using the online convex programming framework.
We prove the following key result for AdaRem.

\begin{theorem}
\label{the:AdaRem}
Let $\{\theta_t\}$ be sequences obtained from Algorithm~\ref{alg1},
$\eta_{t,i}=\frac{\eta}{\sqrt{t}}\left(1+\lambda^t\frac{{g}_{t,i} \times {m}_{t,i}}{\left|{g}_{t,i}\right| \max \left|\boldsymbol{m}_{t}\right|+\epsilon}\right)$ and $\gamma=0$. Assume that $\|\boldsymbol{x} - \boldsymbol{y}\|_\infty \leq D_\infty$ for all $\boldsymbol{x}, \boldsymbol{y} \in \mathcal{F}$ and $\|\nabla f_t(\boldsymbol{x})\| \leq G_2$ for all $t \in [T]$ and $\boldsymbol{x} \in \mathcal{F}$.
For $\theta_t$ generated using the AdaRem algorithm, we have the following bound on the regret
\[
    R_T \leq \frac{D_\infty^2d}{\eta(1-\lambda)^3}\Bigg[(5-4\lambda)\sqrt{T}+2\lambda-1\Bigg]+\frac{D_\infty^2d}{2\eta(1-\lambda)}+ G_2^2d\eta(2\sqrt{T}-1).
\]
\end{theorem}
It is easy to see that the regret of AdaRem is upper bounded by $O(\sqrt{T})$. Please see Appendix for details of the proof of convergence.

We end this section with a comparison to the previous work.
Using momentum to reduce oscillations is also found in Adam-like algorithms and SGD with Momentum. These methods directly use momentum to replace the gradient while our method considering momentum as a representation of the update trend of parameters.
AdaRem inspects whether the momentum is in the same direction with the current gradient for each parameter, so as to carefully adjust the learning rate. 

\begin{figure}[h] 
\centering 
\begin{minipage}{0.47\textwidth} 
\centering 
\begin{algorithm}[H]
  \caption{AdaRem Algorithm}
  \label{alg1}
  \begin{algorithmic}
  \REQUIRE learning rate $\hat{\eta_{t}}$ at each iteration, momentum parameter $\beta$, iteration number $T$, weight decay factor $\gamma$, $m_0=0$, $\lambda$.
  \FOR{$t=0$ to $T$}
  \STATE $\boldsymbol{g}_{t}=\nabla_{\boldsymbol{\theta}} f_{t}\left(\boldsymbol{\theta}_{t}\right)$
  \STATE $\eta_{t,i}=\left(1+\lambda^t\frac{{g}_{t,i} \times {m}_{t,i}}{\left|{g}_{t,i}\right| \max \left|\boldsymbol{m}_{t}\right|+\epsilon}\right) \hat{\eta_{t}}$
  \STATE $\boldsymbol{\theta}_{t+1}=\Pi_{\mathcal{F}, \mathrm{diag}(\eta_t^{-1})}(\boldsymbol{\theta}_{t}-\boldsymbol{\eta}_{t} \odot \boldsymbol{g}_{t}-\hat{\eta_{t}} \gamma \boldsymbol{\theta}_{t} )$
  \STATE $\boldsymbol{m}_{t+1}=\beta \boldsymbol{m}_{t}+(1-\beta) \boldsymbol{g}_{t}$
  \ENDFOR
  \end{algorithmic}
\end{algorithm}
\end{minipage} \hfill
\begin{minipage}{0.47\textwidth} 
\centering 
\captionof{table}{Final test accuracy of various networks on the ImageNet dataset. The bold number indicates the best result.}
\label{tabel2}
\begin{tabular}[H]{ccc}
\toprule
\multirow{2}{*}{Model} & \multicolumn{2}{c}{Top-1 Accuracy($\%$)} \\ \cline{2-3} 
                       & SGDM          & AdaRem-S                \\ \hline
ResNet50               & \textbf{76.18}   & 76.10                    \\
ResNet18               & 70.67         & \textbf{70.89}                    \\\hline
MobileNetV2-1.0        & 70.71         & \textbf{71.71}                    \\ 
MobileNetV2-0.5        & 62.99         & \textbf{64.01}                    \\\hline
ShuffleNetV2-1.0       & 67.37         & \textbf{68.33}                    \\ 
ShuffleNetV2-0.5       & 57.75         & \textbf{60.15}                    \\ 
\bottomrule
\end{tabular}

\end{minipage} 

\end{figure}

\subsection{AdaRem-S}
For networks with Batch Normalization layer or BN \citep{ioffe2015batch}, all the parameters before BN layer satisfy the property of \textbf{Scale Invariance} \citep{li2019exponential}: If for any $c \in \mathbb{R}^{+}, L(\boldsymbol{\theta})=L(c \boldsymbol{\theta})$, then
\begin{enumerate}
\item[(1)]$\left\langle\nabla_{\boldsymbol{\theta}} L, \boldsymbol{\theta}\right\rangle=0$
\item[(2)]$\left.\nabla_{\boldsymbol{\theta}} L\right|_{\boldsymbol{\theta}=\boldsymbol{\theta}_{0}}=\left.c \nabla_{\boldsymbol{\theta}} L\right|_{\boldsymbol{\theta}=c \boldsymbol{\theta}_{0}}$ for any $c>0$
\end{enumerate}

We use the moving average of the gradient to represent the update trend. However, the length of the parameter vector changes while training due to weight decay. In terms of the network containing the BN layer, changing the length of the parameters vector will affect the length of the gradient vector due to the scale invariance, hence makes $\boldsymbol{m}_{t}$ not a good representative of the update trend in the past.

Based on scale invariance, \citet{li2019exponential} proves that weight decay can be seen as an exponentially increasing learning rate schedule. This equivalence holds for BN, which is ubiquitous and provides benefits in optimization and generalization across all standard architectures. This means that weight decay is redundant in training, and thus we can use a learning rate schedule to achieve the same effect. Furthermore, we can fix the length of the parameter vector during the training and optimize the neural network on the sphere, eliminating the influence of the change of the parameter vector's length on the estimation of the update trend. See algorithm 2 for the pseudo-code of our proposed spherical AdaRem algorithm(AdaRem-S).

For optimization algorithms, the learning rate is a very critical hyper-parameter. Then what kinds of learning rate scheduler should be used on the sphere? First of all, we introduce the equivalent learning rate for spherical stochastic gradient descent: (1) We use the exponential learning rate from \citet{li2019exponential} to replace the weight decay; (2) The equivalent learning rate on the sphere is obtained by using scale invariance.

\begin{figure}[h] 
\centering
\begin{minipage}{0.47\textwidth} 
\centering 
\includegraphics[width=0.45\textwidth]{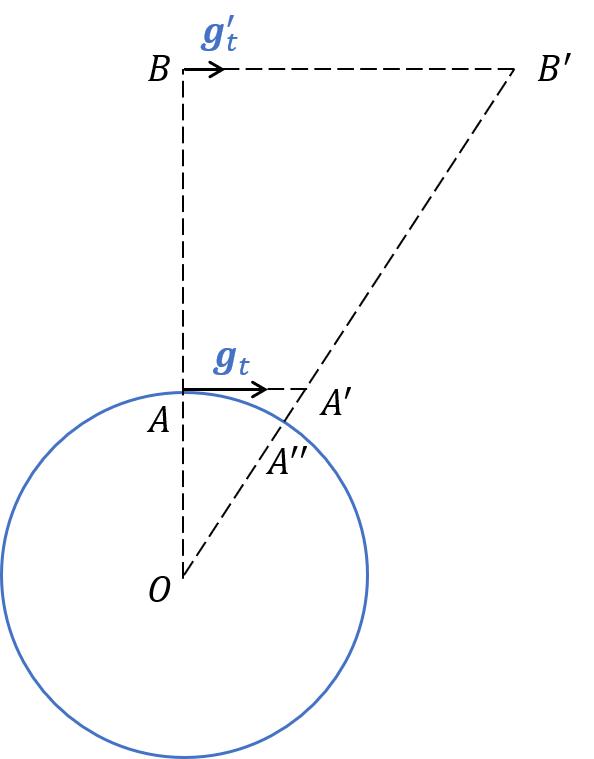} 
\vspace{0.1cm}
\caption{Schematic diagram of spherically constrained optimization.
The length of the solid arrows shows the magnitude of the gradient vectors. 
According to \textbf{Scale Invariance}, network at $A$ and network at $B$ are 
equivalent because they have the same output for any input. Similarly, 
network at ${A}^{'}$ and network at ${B}^{'}$ are equivalent. 
In addition, $\boldsymbol{g}_{t}$ is larger than $\boldsymbol{g}_{t}^{'}$ 
and they are both perpendicular to $OB$.
After the update, $A^{'}$ is projected onto the point $A^{''}$ on the sphere.} 
\label{Fig.mtd2.0} 
\end{minipage} \hfill
\begin{minipage}{0.47\textwidth} 
\centering 
\begin{algorithm}[H]
  \caption{AdaRem-S Algorithm}
  \label{alg2}
  \begin{algorithmic}
  \REQUIRE learning rate $\eta_{t}$ at each iteration, momentum parameter $\beta$, iteration number $T$, initial parameters $\boldsymbol{\theta}_{0}$, sphere radius $R$, weight decay factor $\gamma$, $m_0=0$.
  \STATE $\widehat{\boldsymbol{\theta}}_{0}=\frac{\boldsymbol{\theta}_{0}}{\left\|\boldsymbol{\theta}_{0}\right\|} R$
  \FOR{$t=0$ to $T$}
  \STATE $\boldsymbol{g}_{t}=\nabla_{\boldsymbol{\theta}} f_{t}\left(\hat{\boldsymbol{\theta}}_{t}\right)$
  
  \STATE $\eta_{t,i}=\left(1+\lambda^t\frac{{g}_{t,i} \times {m}_{t,i}}{\left|{g}_{t,i} \right|\max \left|\boldsymbol{m}_{t}\right|+\epsilon}\right) \hat{\eta_{t}}$
  \STATE $\boldsymbol{\theta}_{t+1}=\Pi_{\mathcal{F}, \mathrm{diag}(\eta_t^{-1})}(\hat{\boldsymbol{\theta}}_{t}-\boldsymbol{\eta}_{t} \odot \boldsymbol{g}_{t}-\hat{\eta_{t}} \gamma \hat{\boldsymbol{\theta}}_{t} )$
  \STATE $\boldsymbol{m}_{t+1}=\beta \boldsymbol{m}_{t}+(1-\beta) \boldsymbol{g}_{t}$
  \STATE $\hat{\boldsymbol{\theta}}_{t+1}=\frac{\boldsymbol{\theta}_{t+1}}{\left\|\boldsymbol{\theta}_{t+1}\right\|} R$
  \ENDFOR
  \end{algorithmic}
\end{algorithm}
\end{minipage} 
\end{figure}

\subsubsection{The appropriate learning rate on the sphere}
After constraining the optimization of the neural network to the
sphere, a simple method called SLR(sphere learning rate) is proposed to find the equivalent learning rate on the sphere for SGD.
The parameter vector of the network corresponds to a point in the high dimensional space. Therefore, the term of point in high dimensional space appearing in the following text is a neural network. As shown in Figure \ref{Fig.mtd2.0}, suppose that there is a network at point $B$ and the equivalent network of $B$ on the sphere is at point $A$. According to \textbf{Scale Invariance}, If \(OB=\alpha_{t} OA,\) then \(\boldsymbol{g}_{t}=\alpha_{t} \boldsymbol{g}_{t}^{\prime}.\)

To make the network on the sphere and the network in Euclidean space equivalent everywhere in the training process, it is required that the corresponding points of the two networks are on the same ray from the origin after updating, that is, the network at point $B$ should move to $B^{\prime}$, and the network at point $A$ should move to $A^{\prime}$. According to the similar triangle theorem, \[B B^{\prime}=\alpha_t A A^{\prime}.\] Assuming the learning rates of networks at point $A$ and point $B$ are $\eta_t$ and $\eta_t^{\prime}$, respectively, then\[\eta_{t}^{\prime} \boldsymbol{g_{t}}^{\prime}=\alpha_{t} \eta_{t} \boldsymbol{g_{t}}.\] We obtain the equivalent learning rate on the sphere: \(\eta_{t}=\frac{1}{\alpha^{2}_{t}} \eta^{\prime}_{t}\). Using the similar triangle theorem, $\alpha_{t}$ can be got sequentially as the training goes on. Empirically, we found that it is also working to use the SLR algorithm to obtain an appropriate learning rate on the sphere for AdaRem. Therefore, for our AdaRem-S method, we simply use SLR to obtain the sphere learning rate while training.

\section{Experiments}

\subsection{CONVOLUTIONAL NEURAL NETWORK}
\subsubsection{Dataset and hyper-parameters tuning}
We test our optimizers on ImageNet dataset which contains roughly 1.28 million training images and 50000 validation images with 1000 categories. To our knowledge, it is particularly challenging for adaptive optimizers to outperform SGD on this large dataset. We show that for large scale dataset, our proposed algorithms still enjoy a fast convergence rate, while the unseen data performance of AdaRem-S outperforms SGDM in small networks and much better than existing adaptive optimizers such as Adam, AdaBound and RMSProp.

A learning rate scheduler is crucial to the training. \citet{he2019bag} reports that the cosine learning rate outperforms the step decay learning rate for vision tasks. Therefore, we run all experiments with cosine learning rates without a warmup stage and train for 100 epochs with a minibatch size of 1024 on 16 GPUS. We set the base learning rate of 0.4 for SGD, AdaRem and AdaRem-S, 0.004 for Adam, AdamW and AdaBound, 0.0001 for RMSProp. Empirically we set $\lambda$ of 0.999 for AdaRem and AdaRem-S and just multiplying $\lambda$ once per epoch is enough.
We perform grid searches to choose the best hyper-parameters for all algorithms, additional details can be seen in the Appendix.

\subsubsection{Adaptive optimizers' performance}
We train a ResNet-18 \citep{he2016deep} model on ImageNet with our AdaRem and several commonly used optimizers, including: (1) SGD with momentum(SGDM) \citep{sutskever2013importance}, (2) Adam \citep{kingma2014adam}, (3) AdamW \citep{loshchilov2017decoupled}, (4) AdaBound \citep{luo2019adaptive} and (5) RMSprop \citep{riedmiller1992rprop}. As seen in Figure \ref {Fig.1} and Table \ref{tabel1}, AdamW, AdaBound, RMSProp appear to perform better than SGDM early in training. But at the end of the training, they all have poorer performance on the test set than SGDM. As for our method, AdaRem converges almost fastest and performs as well as SGDM on the test set at the end of training while achieves a significantly minimum train loss.     

\begin{figure}[h]
\centering 
\subfigure[Train Loss for ResNet18]{
\label{Fig1.sub.1}
\includegraphics[width=0.32\textwidth]{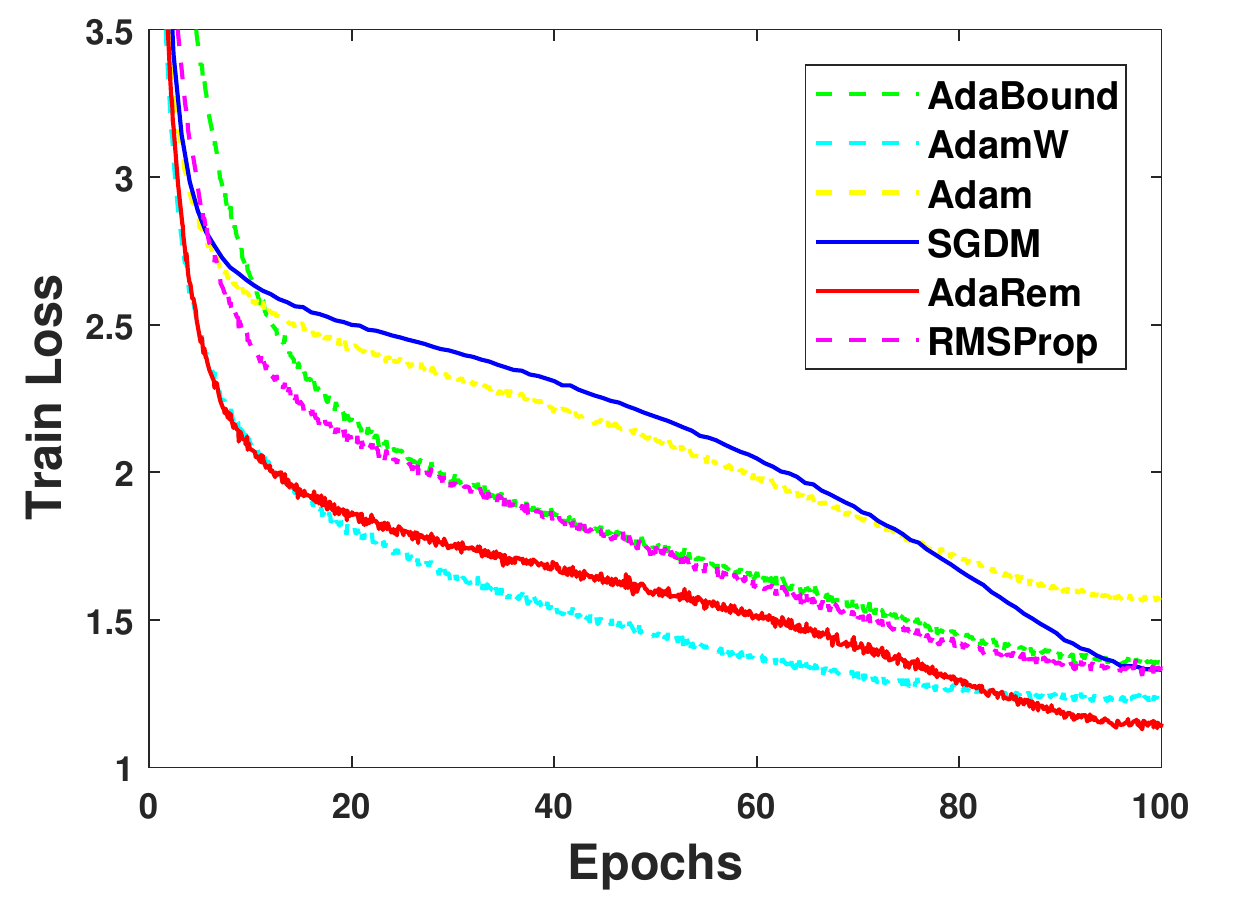}}
\subfigure[Top-1 Error for ResNet18]{
\label{Fig1.sub.2}
\includegraphics[width=0.32\textwidth]{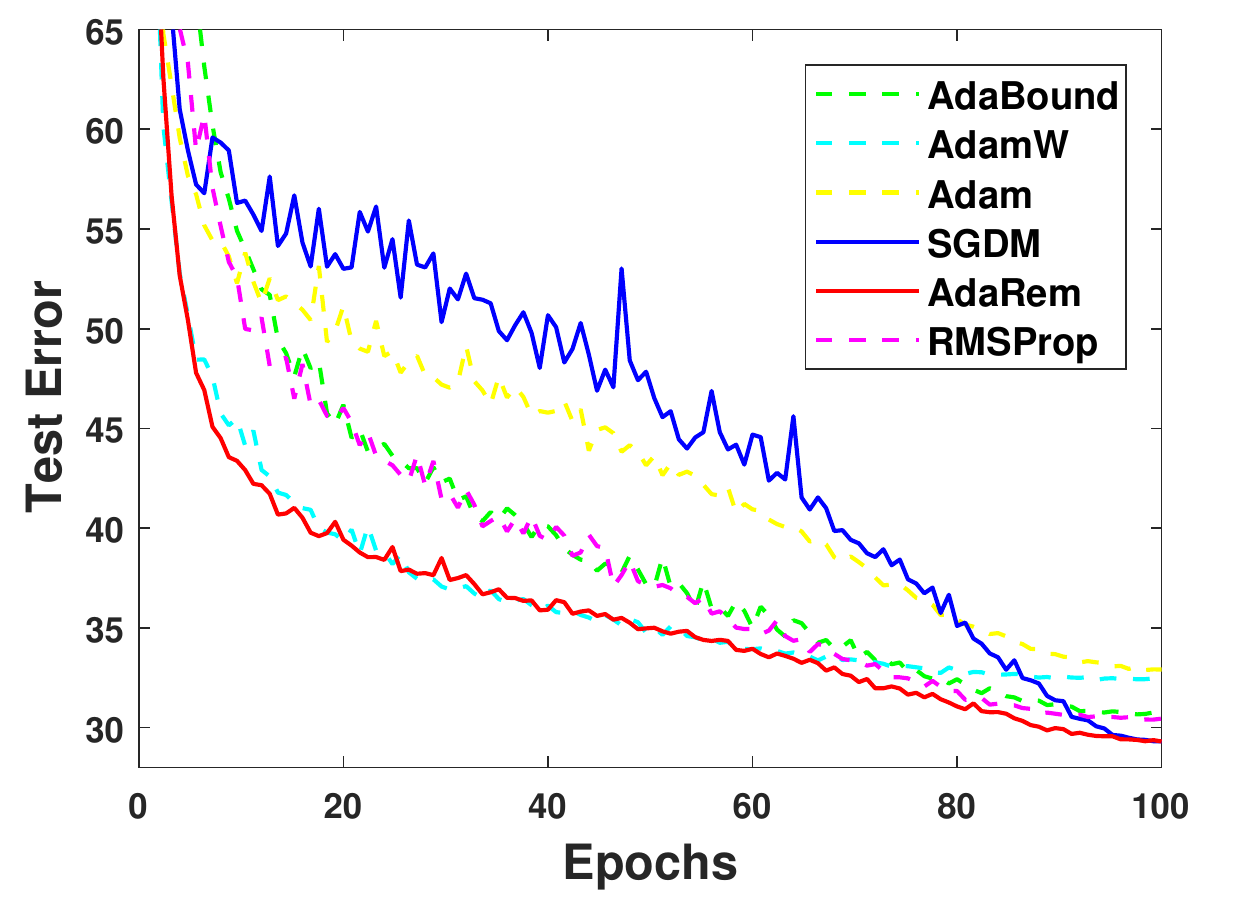}}
\subfigure[Top-5 Error for ResNet18]{
\label{Fig1.sub.3}
\includegraphics[width=0.32\textwidth]{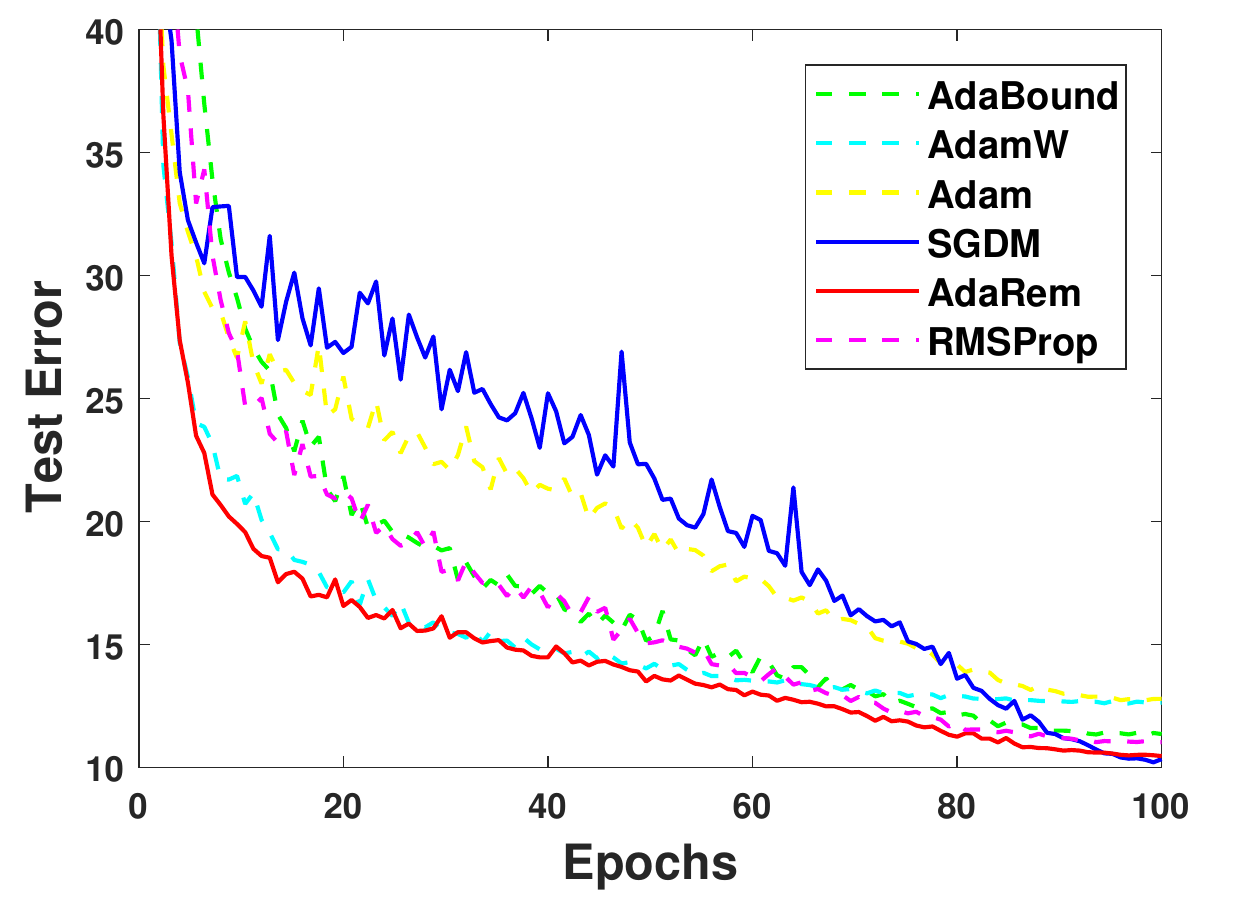}}
\caption{Train loss and test error of ResNet18 on ImageNet. Adam, AdamW, AdaBound and RMSProp have fast progress in the early stages, but their performance quickly enters a period of slow growth. AdaRem achieves the fastest training speed among all methods and performs as well as SGDM on the test set.} 
\label{Fig.1}
\end{figure}

\begin{table}[]
\caption{Final test accuracy of all algorithms on the ImageNet dataset. The bold number indicates the best result.}
\label{tabel1}
\centering
\begin{tabular}{ccccccc}
\toprule
\multirow{2}{*}{Models} & \multicolumn{6}{c}{Test Accuracy($\%$)}                        \\ \cline{2-7} 
                        & SGDM    & Adam               & AdamW         & AdaBound & RMSprop   & AdaRem \\ \hline
ResNet-18 Top-1          & \textbf{70.67}   & 67.07     &  67.55     & 69.3    & 69.63   & \textbf{70.67}   \\
ResNet-18 Top-5          & \textbf{89.74}   & 87.24  &  87.37  & 88.7    & 88.94   & 89.51          \\
\bottomrule
\end{tabular}
\end{table}
\subsubsection{AdaRem-S vs. SGDM across various architectures}
AdaRem-S eliminates the influence of the change of the parameter vector's length on the estimation of momentum. Here we compare AdaRem-S against SGDM on various networks, including ResNet-18, ResNet-50 \citep{he2016deep}, MobileNetV2 \citep{sandler2018mobilenetv2} and ShuffleNetV2 \citep{ma2018shufflenet}. 
\paragraph{ResNet} 
Results for this experiment is shown in Figure \ref {Fig.2}. As expected, AdaRem-S makes rapid progress lowing train loss at the early stage of the training process and finally performs as well as SGDM for ResNet-18 and ResNet-50. From Table \ref{tabel1} and Table \ref{tabel2}, we can see that AdaRem-S performance better than AdaRem in terms of the test accuracy for ResNet-18.
\paragraph{MobileNetV2 and ShuffleNetV2}
As we can see in Figure \ref {Fig.2} and Table \ref{tabel2}, while enjoying a fast convergence rate, AdaRem-S improves the top-1 accuracy by 1.0\% for MobileNetV2 and ShuffleNetV2-1.0. What's more, ShuffleNetV2-0.5 gets a surprising large gain (\textbf{2.4\%}) when using AdaRem-S. Note that AdaRem-S achieves more improvement for smaller models(e.g. MobileNetV2, ShuffleNetV2). This is because the smaller models are more under fitted, and AdaRem-S significantly improves their fitting ability.

\subsection{RECURRENT NEURAL NETWORK}
Finally, to verify the generalization of our optimizer, we trained a Long Short-Term Memory
(LSTM) network \citep{hochreiter1997long} for language modeling task on the Penn Treebank dataset. As BN is not available on a recurrent neural network, we just conduct experiments with AdaRem. We followed the model setup of \citet{merity2017regularizing} and made use of their publicly available code in our experiments. We trained all models for 100 epochs and divided the learning rate by 10 in 50th epoch. For hyper-parameters such as learning rate $\eta$ and parameter $\beta_1$ and $\beta_2$, we performed grid searches to choose the best one. As shown in Table~\ref{tab-lstm}, our method AdaRem has the smallest perplexity value.

\begin{figure}[h]
\centering 
\subfigure[Train Loss for MobileNetV2]{
\label{Fig2.sub.1}
\includegraphics[width=0.32\textwidth]{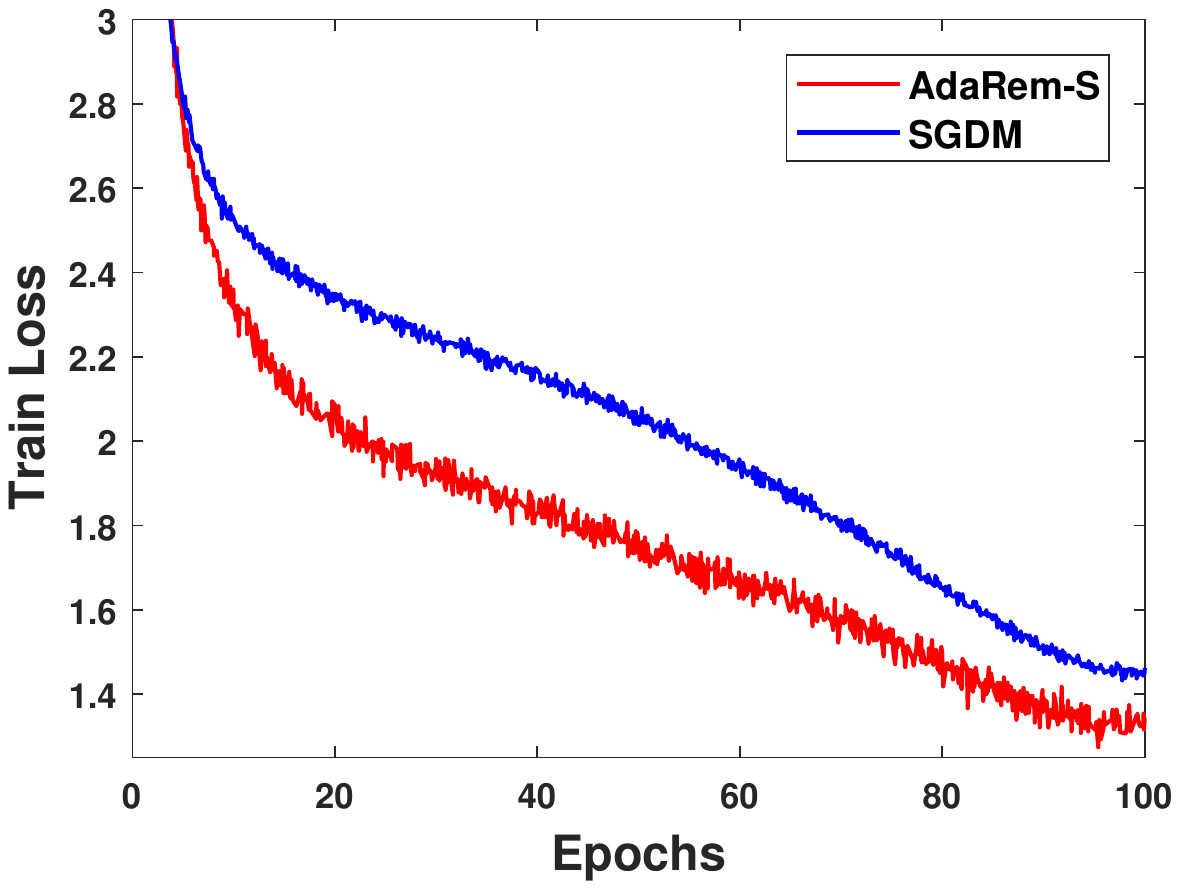}}
\subfigure[Train Loss for ResNet18]{
\label{Fig2.sub.4}
\includegraphics[width=0.32\textwidth]{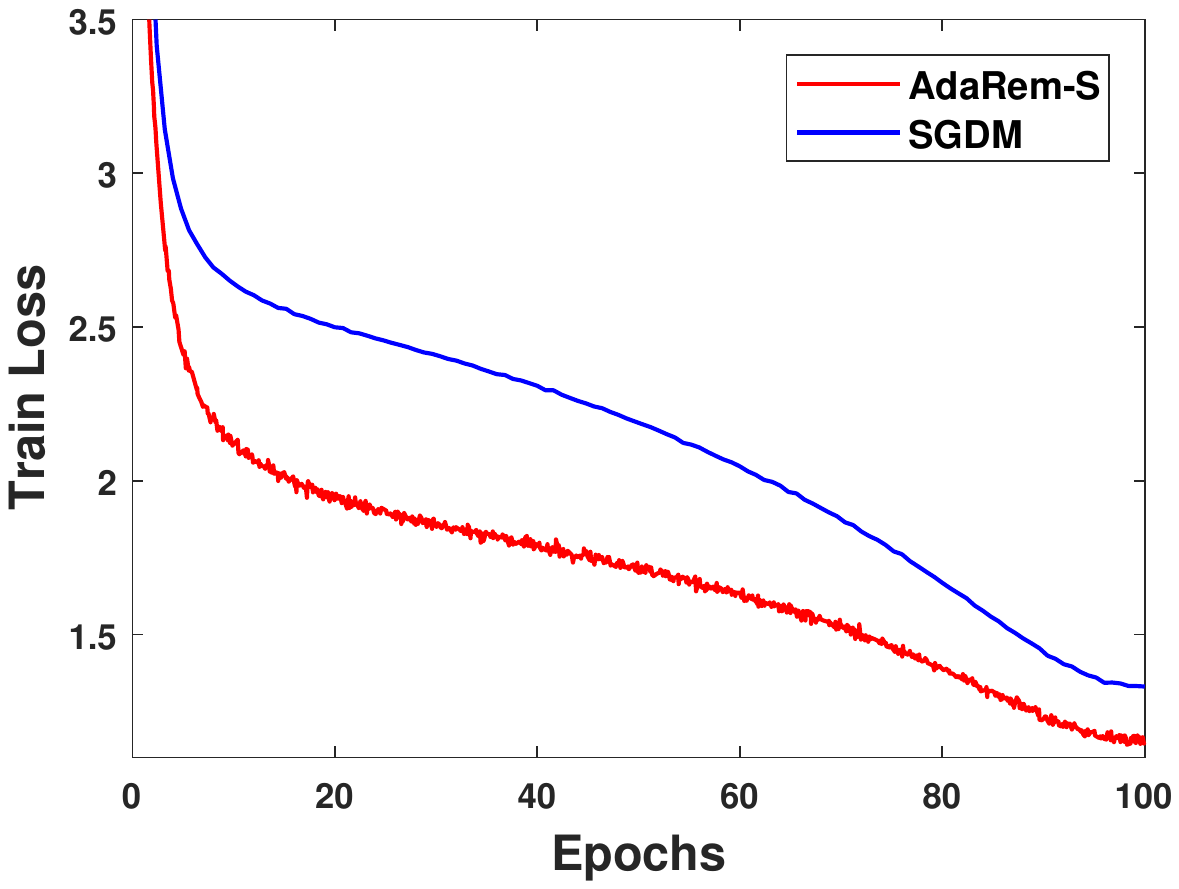}}
\subfigure[Train Loss for ResNet50]{
\label{Fig2.sub.7}
\includegraphics[width=0.32\textwidth]{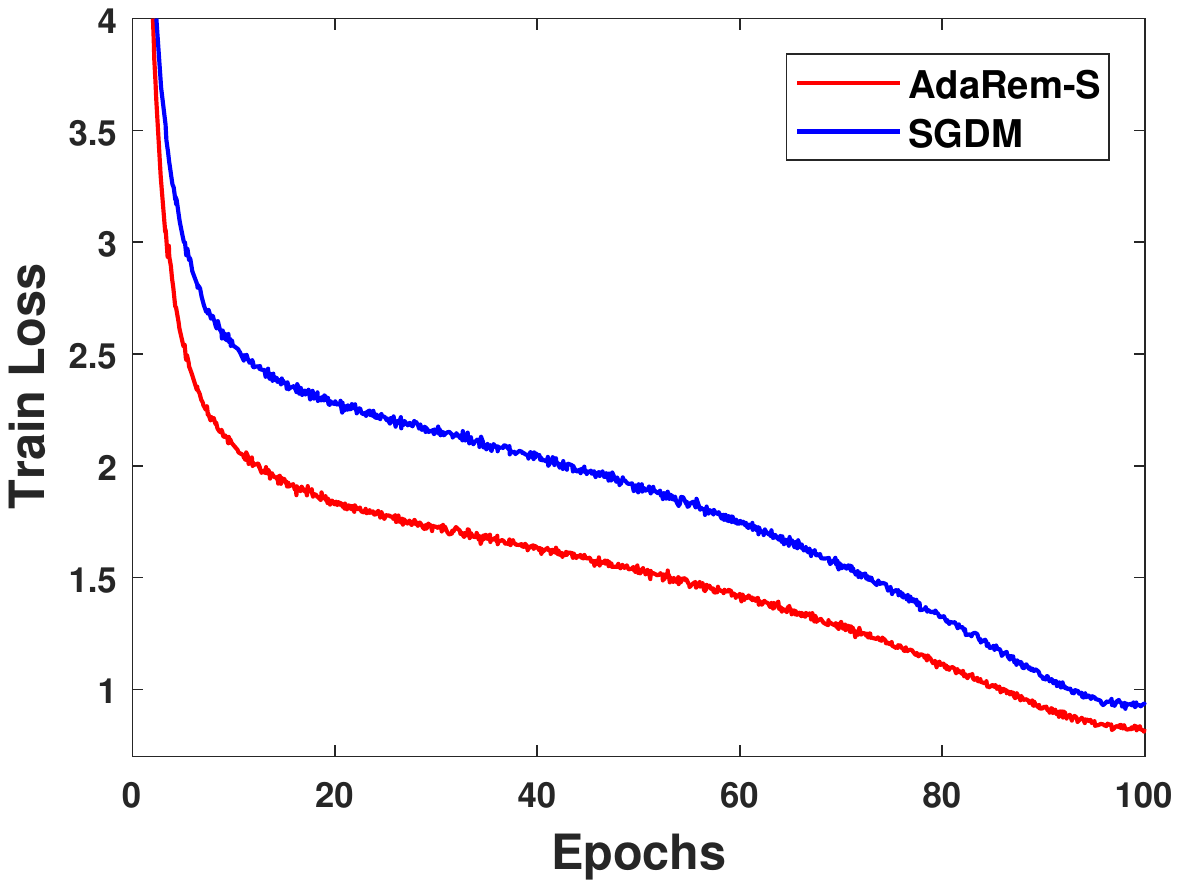}}
\subfigure[Top-1 Error for MobileNetV2]{
\label{Fig2.sub.2}
\includegraphics[width=0.32\textwidth]{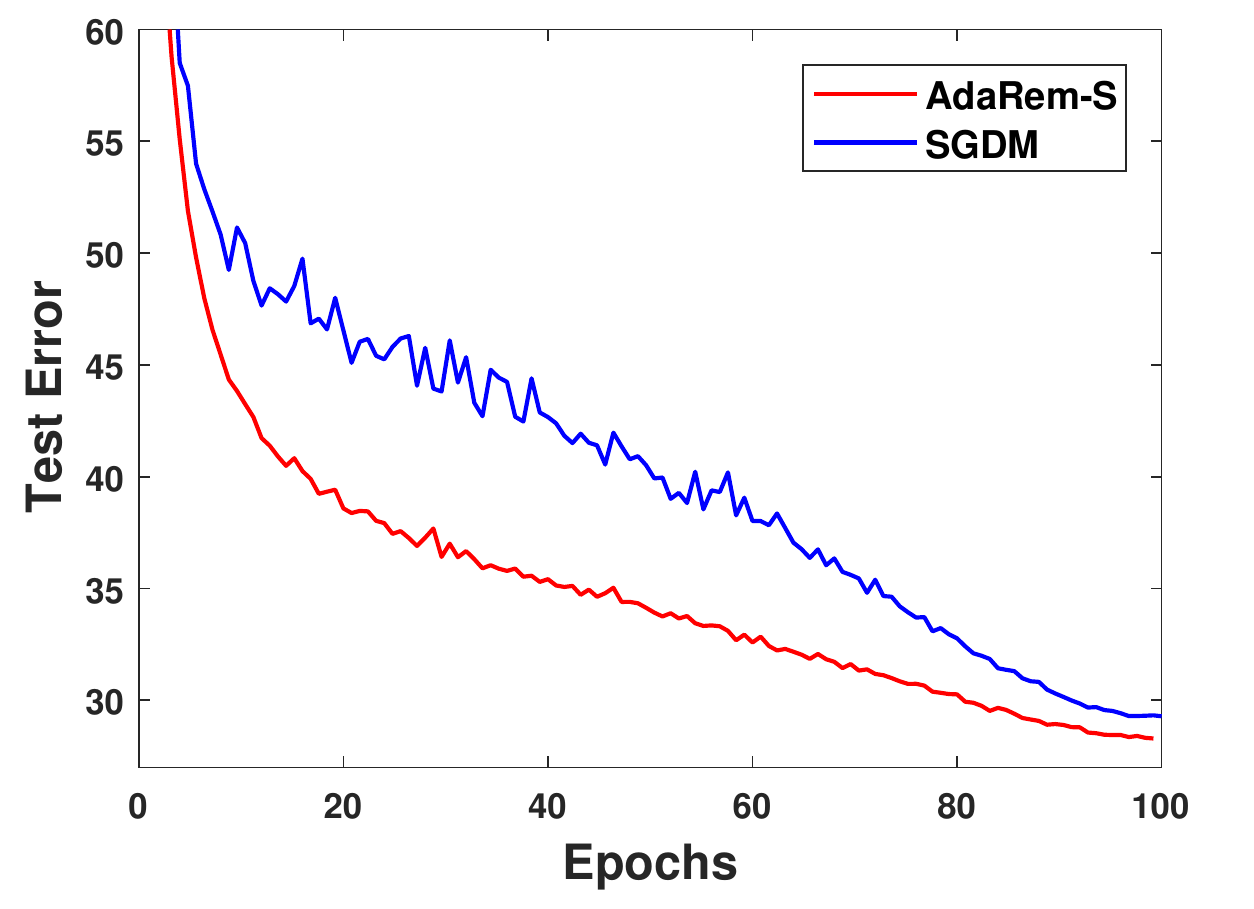}}
\subfigure[Top-1 Error for ResNet18]{
\label{Fig2.sub.5}
\includegraphics[width=0.32\textwidth]{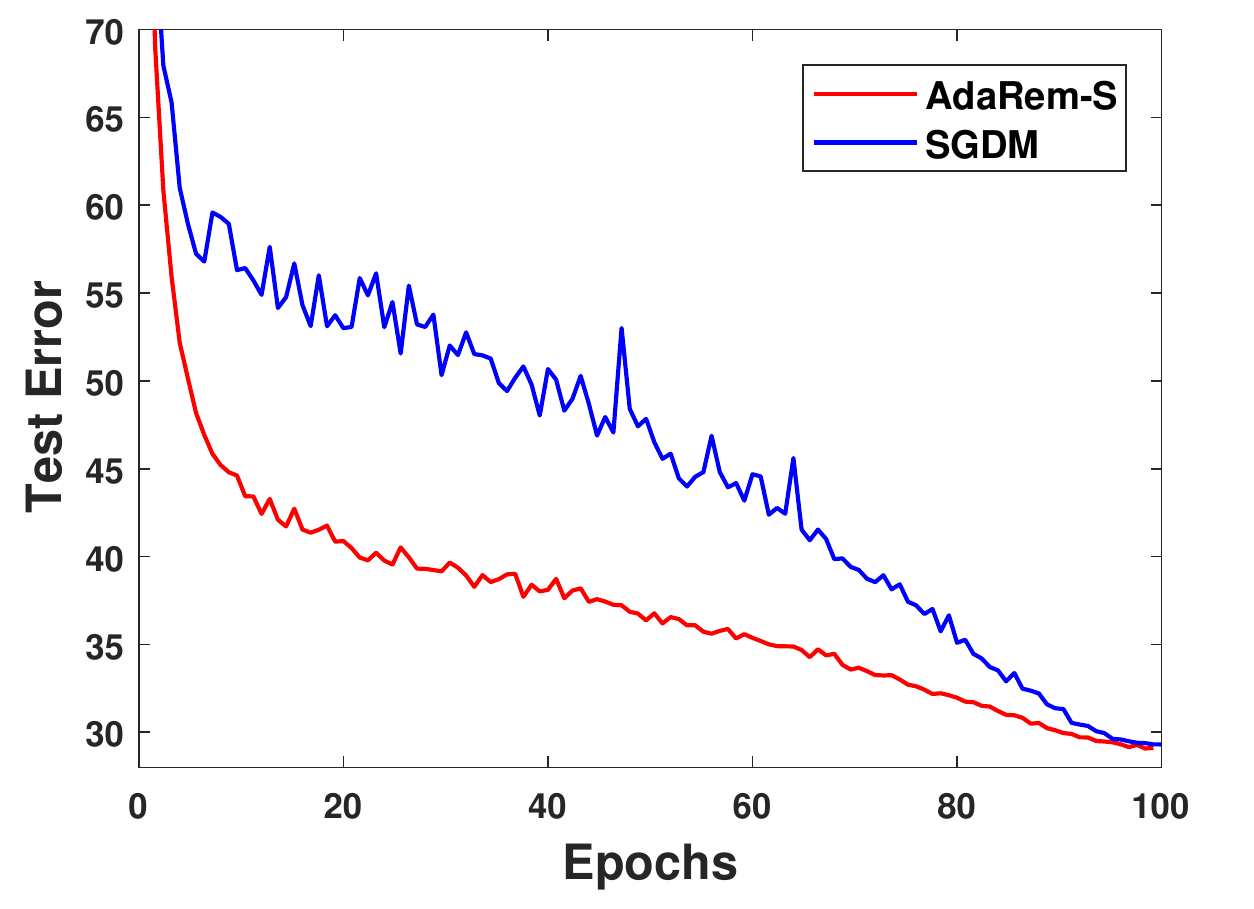}}
\subfigure[Top-1 Error for ResNet50]{
\label{Fig2.sub.8}
\includegraphics[width=0.32\textwidth]{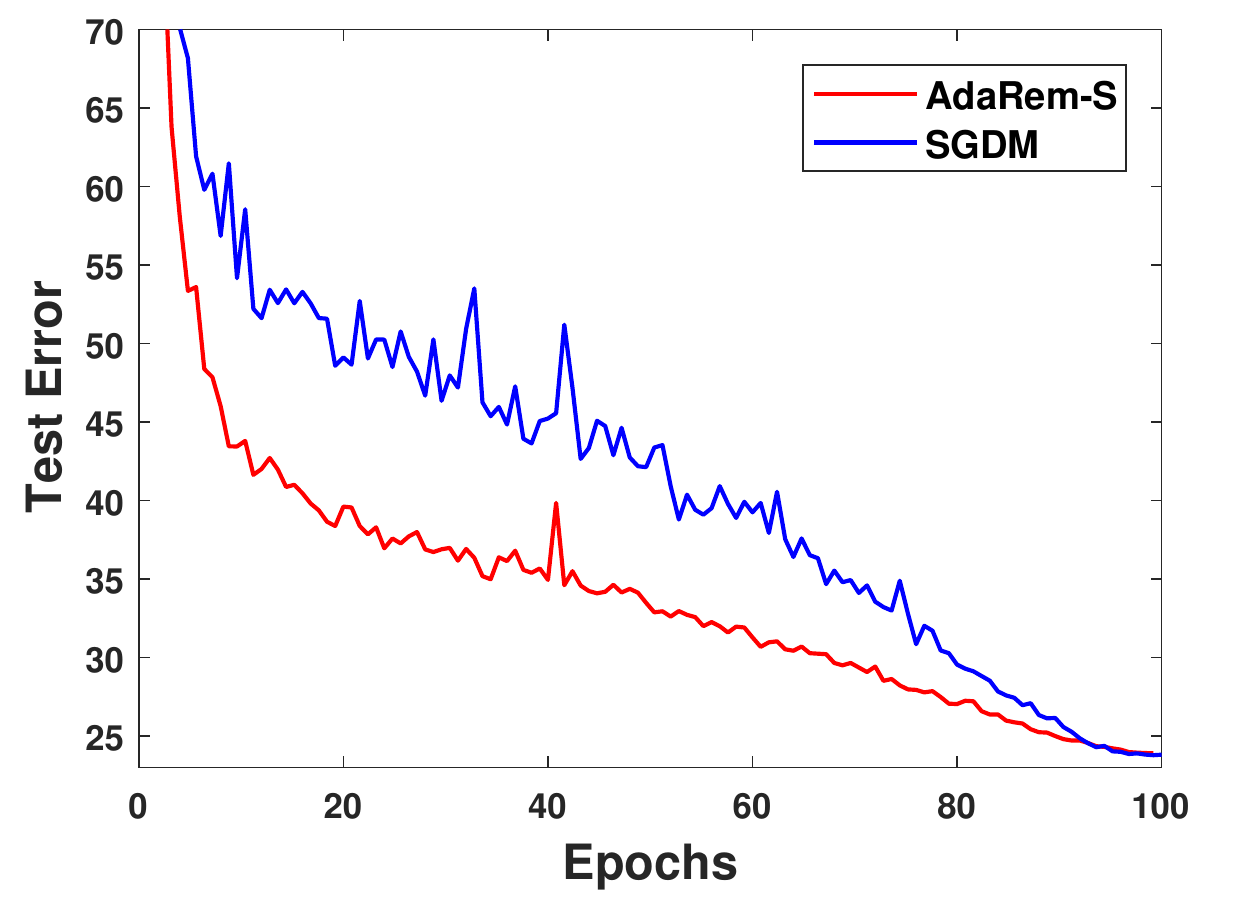}}

\caption{Train loss and test error of three networks on ImageNet. Compared with SGDM, AdaRem-S significantly reduces the training loss across all three networks. It generalizes as well as SGDM on ResNet18 and ResNet50 and brings considerable improvement over SGDM on MobileNetV2.}
\label{Fig.2}
\end{figure}

\begin{table}[]
\caption{Test perplexity of LSTM models on the Penn Treebank dataset. A lower value is better.}
\label{tab-lstm}
\centering
\begin{tabular}{llll}
\toprule
Optimizers & SGDM  & Adam  & AdaRem         \\ \hline
Perplexity & 71.56 & 70.16 & \textbf{69.13} \\ 
\bottomrule
\end{tabular}
\end{table}

\section{Conclusion}
In this paper, we propose AdaRem and its spherical version called AdaRem-S. By changing the learning rate according to whether the direction of the current gradient is the same as the direction of parameter change in the past, these algorithms can accelerate SGD and dampen oscillations more efficiently than SGDM. The experiments show that AdaRem and AdaRem-S can maintain a fast convergence rate while performing as well as SGDM on the unseen data. In particular, AdaRem-S achieves better test performance than SGDM on MobileNetV2 and ShuffleNetV2.

\clearpage  
\bibliography{iclr2021_conference}
\bibliographystyle{iclr2021_conference}

\clearpage  
\appendix
\section{proof of Theorem~\ref{the:AdaRem}}

\begin{lemma}[\citet{mcmahan2010adaptive}]
\label{lem1}
For any $Q \in \mathcal{S}_+^d$ and convex feasible set $\mathcal{F} \subset \R^d$, suppose $u_1 = \min_{x \in \mathcal{F}} \|Q^{1/2}(x-z_1)\|$ and $u_2 = \min_{x \in \mathcal{F}} \|Q^{1/2}(x-z_2)\|$ then we have $\|Q^{1/2}(u_1-u_2)\| \leq \|Q^{1/2}(z_1-z_2)\|$. 
\end{lemma}
\begin{proof}
We provide the proof here for completeness. 
Since $u_{1} = min_{x \in \mathcal{F}} \left \| Q^{1/2}(x-z_{1}) \right \|$ and $u_{2} = min_{x\in \mathcal{F}} \left \| Q^{1/2}(x-z_{2}) \right \|$ and from the property of projection operator we have the following:
\[
    \left \langle z_{1} - u_{1}, Q(z_{2} - z_{1}) \right \rangle \geq 0  
    \text{ and }  
    \left \langle z_{2} - u_{2}, Q(z_{1} - z_{2}) \right \rangle \geq 0.
\]
Combining the above inequalities, we have
\begin{equation}
\label{lem:in}
\left \langle u_{2} - u_{1}, Q(z_{2} - z_{1}) \right \rangle \geq \left \langle z_{2} - z_{1}, Q(z_{2} - z_{1}) \right \rangle.
\end{equation}
Also, observe the following:
\[
\left \langle u_{2} - u_{1}, Q(z_{2} - z_{1}) \right \rangle \leq \frac{1}{2}\left [ \left \langle u_{2} - u_{1}, Q(u_{2} - u_{1}) \right \rangle + \left \langle z_{2} - z_{1}, Q(z_{2} - z_{1}) \right \rangle \right ].
\]
The above inequality can be obtained from the fact that 
\[
\left \langle (u_{2} - u_{1}) - (z_{2} - z_{1}), Q((u_{2} - u_{1}) - (z_{2} - z_{1})) \right \rangle \geq 0 \text{ as } Q \in \mathcal{S}_{+}^{d}
\]
and rearranging the terms. 
Combining the above inequality with \Eqref{lem:in}, we have the required the result.
\end{proof}

\begin{proof}
For simplicity, vectors is also denoted in common lowercase in the proof.
We begin with the following observation:
\[
    x_{t+1} = \Pi_{\mathcal{F}, \mathrm{diag}(\eta_t^{-1})} (x_t - \eta_t \odot g_t) = \min_{x \in \mathcal{F}} \|\eta_t^{-1/2} \odot (x -  (x_t - \eta_t \odot g_t))\|.
\]
Furthermore, as $\mathcal{F}$ is closed and convex, we can get $x^* = \argmin_{x \in \mathcal{F}} \sum_{t=1}^{T} f_t(x)$. Using Lemma~\ref{lem1} with $u_1 = x_{t+1}$ and $u_2 = x^*$, we have the following:
\begin{equation*}
\begin{split}
    \|\eta_t^{-1/2} \odot (x_{t+1} - x^*)\|^2
        &\leq \|\eta_t^{-1/2} \odot (x_t - \eta_t \odot g_t - x^*)\|^2 \\
        &= \|\eta_t^{-1/2} \odot (x_t - x^*)\|^2 + \|\eta_t^{1/2} \odot g_t\|^2 - 2\langle g_t, x_t - x^* \rangle.
\end{split}
\end{equation*}
Rearranging the above inequality, we have
\begin{equation}
\label{eq:inner-product-bound}
\begin{split}
    \langle g_t, x_t - x^* \rangle
        &\leq \frac{1}{2}\bigg[ \|\eta_t^{-1/2} \odot (x_t - x^*)\|^2 - \|\eta_t^{-1/2} \odot (x_{t+1} - x^*)\|^2 \bigg] + \frac{1}{2} \|\eta_t^{1/2} \odot g_t\|^2.
\end{split}
\end{equation}

We now use the standard approach of bounding the regret at each step using convexity of the function $f_t$ in the following manner:
\begin{equation}
\label{eq:regret-1}
\begin{split}
    &\sum _ { t = 1 } ^ { T } f _ { t } \left( x _ { t } \right) - f _ { t } \left( x ^ { * } \right) \leq \sum _ { t = 1 } ^ { T } \left\langle g _ { t } , x _ { t } - x ^ { * } \right\rangle \\
    &\leq \frac{1}{2} \sum_{t = 1}^{T} \bigg[ \|\eta_t^{-1/2} \odot (x_t - x^*)\|^2 - \|\eta_t^{-1/2} \odot (x_{t+1} - x^*)\|^2  + \|\eta_t^{1/2} \odot g_t\|^2 \bigg] \\
    &= \frac{1}{2} \Bigg[ \sum_{t = 2}^{T} \bigg[ \|\eta_t^{-1/2} \odot (x_t - x^*)\|^2 - \|\eta_{t-1}^{-1/2} \odot (x_t - x^*)\|^2 \bigg] \\
    & \quad\quad + \|\eta_1^{-1/2} \odot (x_1 - x^*)\|^2 - \|\eta_{t}^{-1/2} \odot (x_{t+1} - x^*)\|^2 + \sum_{t = 1}^{T} \|\eta_t^{1/2} \odot g_t\|^2 \Bigg]\\
    &= \frac{1}{2} \Bigg[ \sum_{t=2}^T \sum_{i=1}^d (x_{t,i} - x_i^*)^2 (\eta_{t,i}^{-1} - \eta_{t-1,i}^{-1}) +\sum_{i=1}^d \eta_{1,i}^{-1} (x_{1,i} - x_i^*)^2 \\
    & \quad\quad - \sum_{i=1}^d \eta_{t,i}^{-1} (x_{t+1,i} - x_i^*)^2 + \sum_{t=1}^T \sum_{i=1}^d g_{t,i}^2 \eta_{t,i} \Bigg].
\end{split}
\end{equation}
The first inequality is due to the convexity of functions $\{f_t\}$.
The second inequality follows from the bound in \Eqref{eq:inner-product-bound}.
For further bounding this inequality, we need the following intermediate result.

\begin{lemma}
\label{lem2}
For the parameter settings and conditions assumed in Algorithm~\ref{alg1}, we have
\[
    \sum_{t=2}^T\left| \eta_{t,i}^{-1} - \eta_{t-1,i}^{-1}\right|\leq \frac{2}{\eta(1-\lambda)^3}\Bigg[(5-4\lambda)\sqrt{T}+2\lambda-1\Bigg].
\]
\end{lemma}
\begin{proof}
For simplicity, we ignore subscript $i$ in this lemma. Let $\eta_{t+1} = \eta_{t}+ \Delta t$, $c_t=\frac{g_t m_t}{\left|g_t^i \right| \max \left(\left|m_t\right|\right)+\epsilon}$, $b_t=\lambda^t c_t$ and $a_t=\frac{1}{\sqrt{t}}$, thus we have
\[
    \eta_t=\frac{\eta}{\sqrt{t}}(1+b_t),
\]
\begin{equation*}
\begin{split}
    & \frac{\Delta t}{\eta} = \frac{\eta_{t+1}-\eta_t}{\eta} = \frac{1}{\sqrt{t+1}}\left(1+b_{t+1} \right) - \frac{1}{\sqrt{t}}\left(1+b_t \right) \\
    &= \left(\frac{1}{\sqrt{t+1}} - \frac{1}{\sqrt{t}} \right) + a_{t+1}b_{t+1}-a_t b_t \\
    &=a_{t+1}\left(b_{t+1}-b_t\right)+\left(1+b_t\right)\left(a_{t+1}-a_t\right).
\end{split}
\end{equation*}

We observe that, 
\[
   \left| a_{t+1}\left(b_{t+1}-b_t\right)\right|=\frac{\left|\lambda^{t+1}c_{t+1}-\lambda^t c_t\right|}{\sqrt{t+1}} \\
   \leq \frac{\lambda^{t+1}\left|c_{t+1}\right|+\lambda^t \left|c_t\right|}{\sqrt{t+1}}\\
   \leq \frac{2\lambda^t}{\sqrt t}.
\]
Also, observe the following: 
\[
    \left|\left(1+b_t\right)\left(a_{t+1}-a_t\right)\right|\leq \left(1+\lambda\right)\left(\frac{1}{\sqrt{t}} - \frac{1}{\sqrt{t+1}} \right)
    \leq \frac{2}{t\sqrt t}.
\]
The above inequality can be obtained from the fact that $\left|1+b_t\right|\leq 1+\left|b_t\right|\leq 1+\lambda \left|c_t\right|\leq 1+\lambda$.

Hence, we have, 
\[
    \frac{\left|\Delta t \right|}{\eta}\leq \frac{2}{\sqrt{t}}(\frac{1}{t}+\lambda^t)
\]

By definition, 
\[
    \eta\frac{1-\lambda}{\sqrt{t}}\leq \eta_t \leq \eta\frac{1+\lambda}{\sqrt{t}}
\]
And then,
\begin{equation*}
\begin{split}
    &\left|\eta_{t+1}^{-1} - \eta_{t}^{-1}\right| = \left|\frac{\eta_{t+1}-\eta_t}{\eta_{t+1}\eta_t}\right| 
    \leq \frac{\left| \Delta t \right|(t+1)}{\eta^2(1-\lambda)^2}
    \leq \frac{2(t+1)(\frac{1}{t}+\lambda^t)}{\sqrt{t}\eta(1-\lambda)^2}.
\end{split}
\end{equation*}
Finally, we have,
\begin{equation*}
\begin{split}
   &\sum_{t=2}^T\left| \eta_{t,i}^{-1} - \eta_{t-1,i}^{-1}\right| 
   \leq \frac{2}{\eta(1-\lambda)^2}\Bigg[\sum_{t=2}^T\frac{t}{t-1}\frac{1}{\sqrt{t-1}}+\sum_{t=2}^T \lambda^{t-1}\frac{t}{\sqrt{t-1}}\Bigg] \\
   & \leq \frac{2}{\eta(1-\lambda)^2}\Bigg[2\sum_{t=2}^T\frac{1}{\sqrt{t-1}}+\sum_{t=2}^T \lambda^{t-1}+\sum_{t=2}^T \lambda^{t-1}\sqrt{t-1}\Bigg] \\
   &\leq \frac{2}{\eta(1-\lambda)^3}\Bigg[(5-4\lambda)\sqrt{T}+2\lambda-1\Bigg].
\end{split}
\end{equation*}
The last inequality is due to the following upper bound:
\begin{equation*}
    \sum_{t=1}^T \frac{1}{\sqrt{t}}
        \leq 1 + \int_{t=1}^T \frac{\mathrm{d}t}{\sqrt{t}}
        = 2\sqrt{T} - 1.
\end{equation*}
\end{proof}

We now return to the proof of Theorem~\ref{the:AdaRem}.
Using the $D_\infty$ bound on the feasible region and making use of the above property in \Eqref{eq:regret-1} and Lemma~\ref{lem2}, we have
\begin{equation*}
\begin{split}
    &\sum _ { t = 1 } ^ { T } f _ { t } \left( x _ { t } \right) - f _ { t } \left( x ^ { * } \right) \\
    &\leq \frac{1}{2} \Bigg[ \sum_{t=2}^T \sum_{i=1}^d (x_{t,i} - x_i^*)^2 \left|\eta_{t,i}^{-1} - \eta_{t-1,i}^{-1}\right| +\sum_{i=1}^d \eta_{1,i}^{-1} (x_{1,i} - x_i^*)^2 + \sum_{t=1}^T \sum_{i=1}^d g_{t,i}^2 \eta_{t,i} \Bigg] \\
    & \leq \frac{D_\infty^2d}{\eta(1-\lambda)^3}\Bigg[(5-4\lambda)\sqrt{T}+2\lambda-1\Bigg]+\frac{D_\infty^2d}{2\eta(1-\lambda)}+ G_2^2d\eta(2\sqrt{T}-1).
\end{split}
\end{equation*}

It is easy to see that the regret of AdaRem is upper bounded by $O(\sqrt{T})$.
\end{proof}

\section{Experiment Details}
\subsection{Hyper-parameters grid search}
We run all experiments with cosine learning rate without a warmup stage and train for 100 epochs with a minibatch size of 1024 on 16 GPUs.

\subsubsection{Adaptive optimizers’ performance on deep convolutional network}
We set the base learning rate of 0.4 for SGDM just as  \citep{he2019bag}, AdaRem and AdaRem-S. For Adam, we set the base learning rate as 0.004, and choose the EMA parameter of the second momentum of gradient $\beta_2$ from \{0.99,0.999\}. For AdamW and AdaBound, we adopt the same hyper-parameters as Adam. For AdaBound, we choose the final\_lr from \{0.1,0.4\}. For RMSProp, we choose the base learning rate from \{0.04,0.004,0.0004\}, and $\beta_2$ from \{0.99,0.999\}. The momentum parameter $\beta$ of AdaRem is set as 0.999. Weight decay parameter is choosen from \{0.0001,0.0003\} for all methods. Additional details can be seen in the Table \ref{tabel3}.

\begin{table} [H] 
\centering
\caption{Hyper-parameters' setting of various optimization methods for ResNet18 on ImageNet. The bold number indicates the best one of the hyper-parameters to be selected and $\epsilon$ is a term to improve numerical stability. $\beta_1$ is the EMA parameter of the first momentum of gradient and $\beta_2$ is the EMA parameter of the second momentum of gradient.}
\label{tabel3}
\begin{tabular}{cccclcc}
\toprule
\multirow{2}{*}{Model} & \multicolumn{6}{c}{Hyper-parameter}                                                                        \\ \cline{2-7} 
                           & lr                                    & $\beta_1$  & $\beta_2$      & \multicolumn{1}{c}{weight\_decay} & $\epsilon$    & final\_lr \\ \hline
SGDM                       & 0.4                                   & {\textbf{0.9},0.999}    & $\sim$     & \textbf{0.0001},0.0003         & $\sim$ & $\sim$     \\
Adam                       & 0.004                                 & {\textbf{0.9},0.999}    & 0.99,\textbf{0.999} & \textbf{0.0001},0.0003          & 1e-8   & $\sim$     \\
AdamW                      & 0.004                                 & {\textbf{0.9},0.999}    & 0.999      & \textbf{0.0001},0.0003          & 1e-8   & $\sim$     \\
AdaBound                   & 0.004                                 & {\textbf{0.9},0.999}    & 0.999     & \textbf{0.0001},0.0003           & 1e-8   & 0.1,\textbf{0.4}    \\
RMSProp                    & \multicolumn{1}{l}{0.04,0.004,\textbf{0.0004}} & $\sim$ & \textbf{0.99},0.999       & \textbf{0.0001},0.0003          & 1e-8   & $\sim$     \\
AdaRem                     & 0.4                                   & {0.9,\textbf{0.999}}    & $\sim$     & 0.0001,\textbf{0.0003}          & 1e-8   & $\sim$     \\
\bottomrule
\end{tabular}
\end{table}

\subsubsection{AdaRem-S vs. SGDM across various architectures}
For AdaRem-S, we choose momentum parameter from \{0.995,0.999\} of ResNet50, and the sphere radius from \{10,100\} of MobileNetV2. MobileNetV2-0.5 and MobileNetV2-1.0 employ the same hyper-parameter's setting. Additional details can be seen in the Table \ref{tabel4}.
\begin{table} [H]
\centering
\caption{Hyper-parameters' setting of AdaRem-S and SGD for ResNet-18, ResNet-50 and MobileNetV2 on ImageNet. The bold number indicates the best one of the hyper-parameters to be selected.}
\label{tabel4}
\begin{tabular}{ccccccc}
\toprule
\multirow{2}{*}{Model} & \multirow{2}{*}{optimizer} & \multicolumn{5}{c}{Hyper-parameter}                          \\ \cline{3-7} 
                               &                            & lr  & momentum             & weight\_decay & R      & $\epsilon$    \\ \hline
\multirow{2}{*}{ResNet-18}      & SGDM                       & 0.4 & 0.9                  & 1e-4          & $\sim$ & $\sim$ \\
                               & AdaRem-S                   & 0.4 & 0.999                & 1e-4          & 10     & 1e-8   \\ \hline
\multirow{2}{*}{ResNet-50}      & SGDM                       & 0.4 & 0.9                  & 1e-4          & $\sim$ & $\sim$ \\
                               & AdaRem-S                   & 0.4 & \textbf{0.995},0.999 & 1e-4          & 10     & 1e-8   \\ \hline
\multirow{2}{*}{MobileNetV2}      & SGDM                       & 0.4 & 0.9                  & 1e-4,\textbf{4e-5}          & $\sim$ & $\sim$ \\
                               & AdaRem-S                   & 0.4 & 0.999                & 1e-4,\textbf{4e-5}          & 10,\textbf{100} & 1e-8   \\ \hline
\multirow{2}{*}{ShuffleNetV2}     & SGDM                       & 0.4 & 0.9                  & 1e-4,\textbf{4e-5}          & $\sim$ & $\sim$ \\
                               & AdaRem-S                   & 0.4 & 0.999                & 1e-4,\textbf{4e-5}          & 10,\textbf{100} & 1e-8   \\
\bottomrule
\end{tabular}
\end{table}

\subsection{Other experimental results}
\begin{figure}[h] 
\centering 
\subfigure[Top-5 Error for MobileNetV2]{
\label{Figa1.sub.1}
\includegraphics[width=0.32\textwidth]{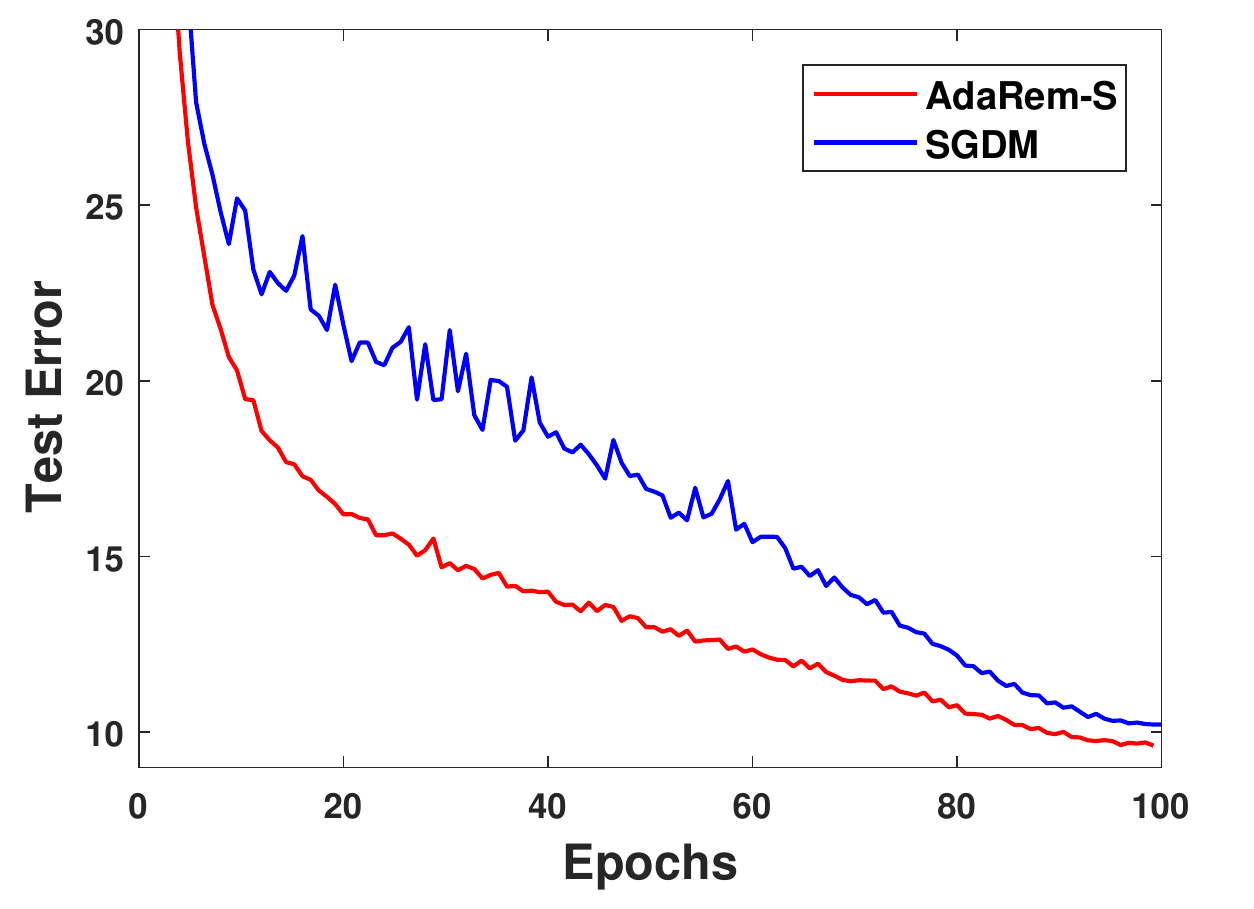}}
\subfigure[Top-5 Error for ResNet18]{
\label{Figa1.sub.2}
\includegraphics[width=0.32\textwidth]{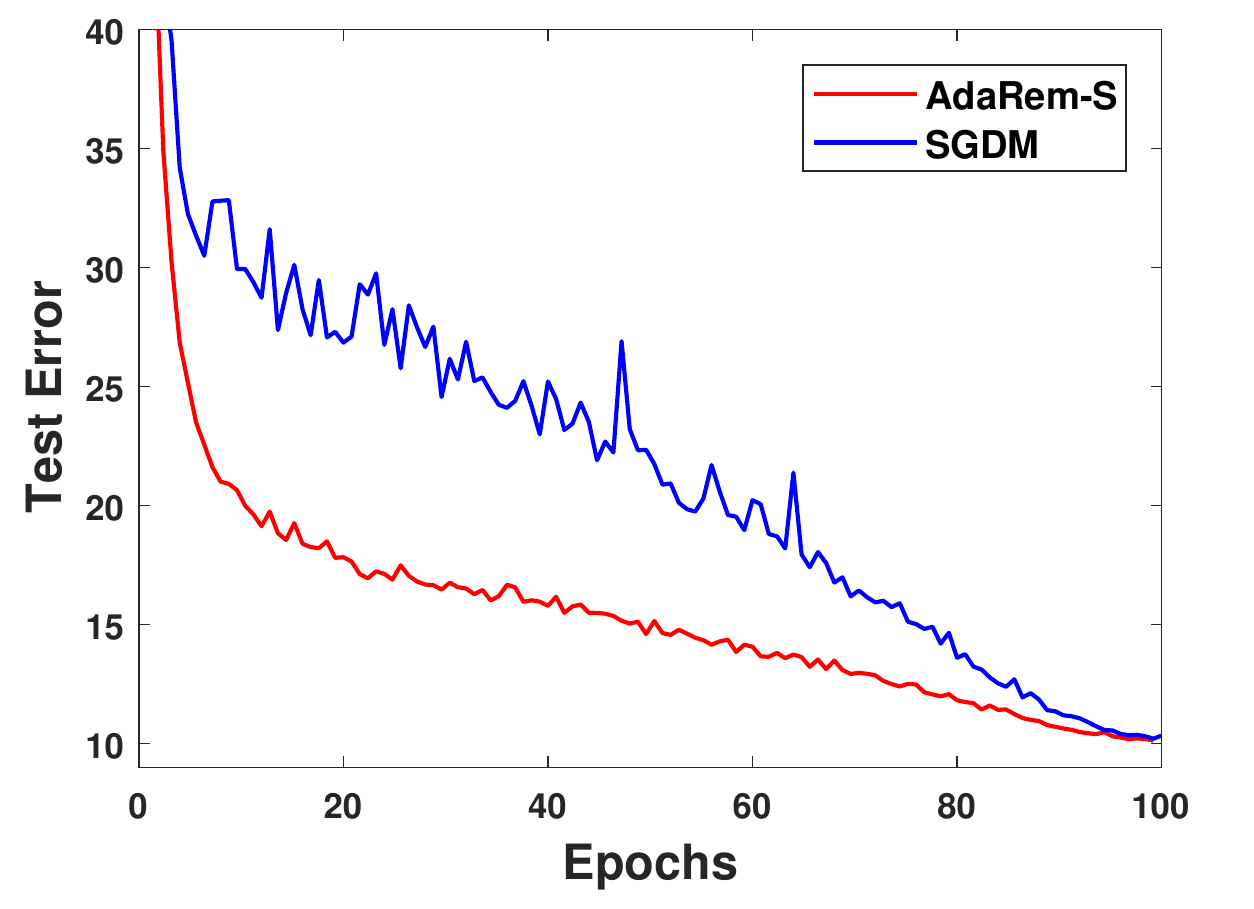}}
\subfigure[Top-5 Error for ResNet50]{
\label{Figa1.sub.3}
\includegraphics[width=0.32\textwidth]{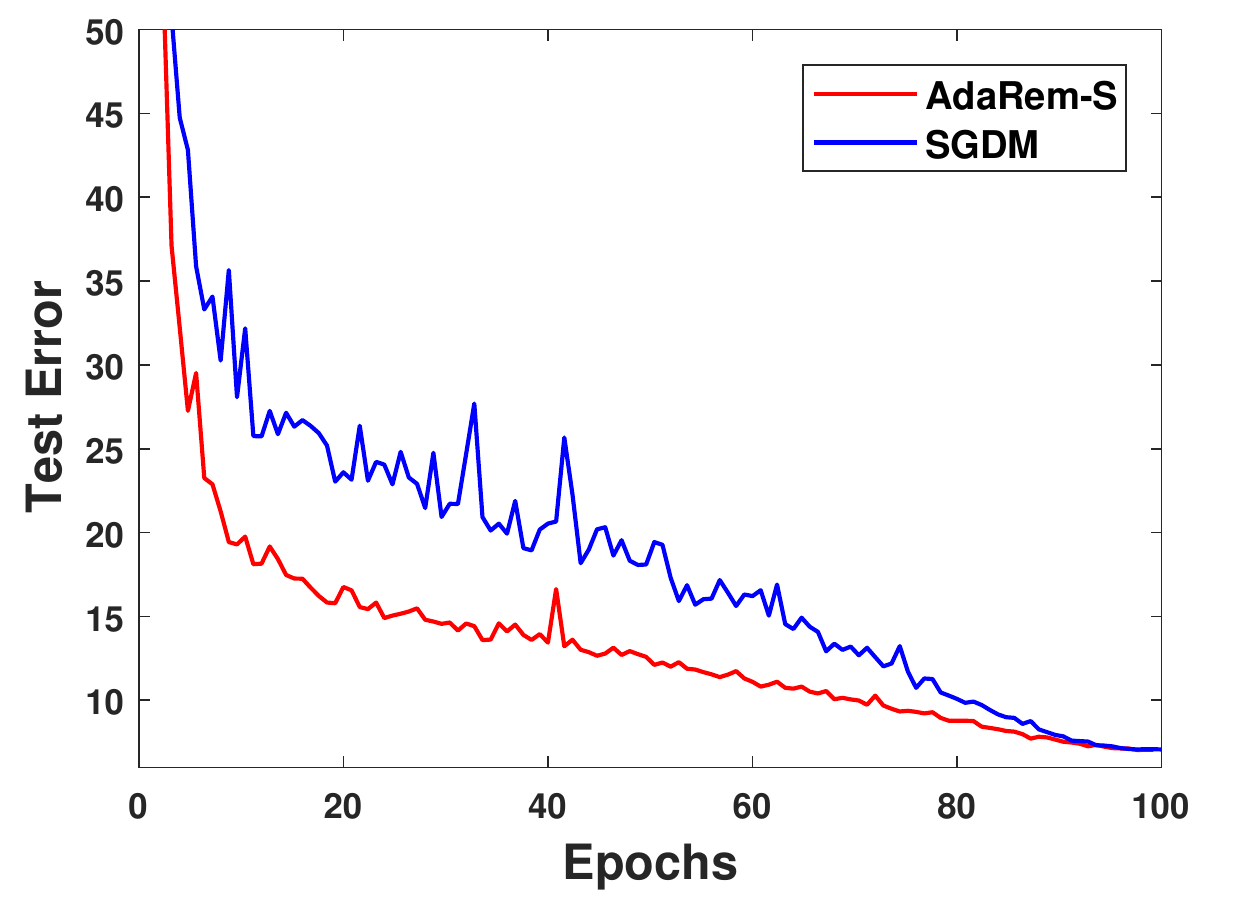}}
\caption{Top-5 error of three networks on ImageNet.}
\label{Fig.a1}
\end{figure}

\begin{table}[H]
\caption{Top-5 accuracy of various networks on the ImageNet dataset. The bold number indicates the best result.}
\label{tabel5}
\centering
\begin{tabular}{ccc}
\toprule
\multirow{2}{*}{Model} & \multicolumn{2}{c}{Top-5 Accuracy($\%$)} \\ \cline{2-3} 
                       & SGDM          & AdaRem-S                \\ \hline
ResNet50               & 92.9          & \textbf{92.92}                    \\
ResNet18               & 89.74         & \textbf{89.87}                    \\ \hline
MobileNetV2-1.0        & 89.81         & \textbf{90.34}                    \\ 
MobileNetV2-0.5        & 84.66         & \textbf{85.30}                    \\ \hline
ShuffleNetV2-1.0        & 87.68         & \textbf{87.96}                    \\ 
ShuffleNetV2-0.5        & 80.23         & \textbf{81.77}                    \\ 
\bottomrule
\end{tabular}
\end{table}

\begin{table}[H]
\caption{Train loss of various networks on the ImageNet dataset. The bold number indicates the best result.}
\label{tabel6}
\centering
\begin{tabular}{ccc}
\toprule
\multirow{2}{*}{Model} & \multicolumn{2}{c}{Train Loss} \\ \cline{2-3} 
                       & SGDM          & AdaRem-S                \\ \hline
ResNet50               & 0.932         & \textbf{0.823}                    \\
ResNet18               & 1.331         & \textbf{1.148}                    \\ \hline
MobileNetV2-1.0        & 1.462         & \textbf{1.343}                    \\ 
MobileNetV2-0.5        & 1.965         & \textbf{1.887}                    \\ \hline
ShuffleNetV2-1.0        & 1.340         & \textbf{1.337}                    \\ 
ShuffleNetV2-0.5        & 1.839         & \textbf{1.744}                    \\ 
\bottomrule
\end{tabular}
\end{table}

\end{document}